\documentclass{article}

\usepackage{times}
\usepackage{graphicx} 
\usepackage{subcaption}

\usepackage{natbib}

\usepackage{algorithm}
\usepackage{algorithmic}

\usepackage{hyperref}

\usepackage[accepted]{icml2018}

\icmltitlerunning{Importance Weighted Transfer of Samples in Reinforcement Learning}

\usepackage[usenames,dvipsnames]{xcolor}
\usepackage{tikz,pgfplots}
\pgfplotsset{compat=newest}

\usepackage{amsfonts}
\usepackage{amsmath}
\usepackage{amsthm}
\usepackage{mathtools}
\usepackage{bm}
\usepackage{multicol}
\usepackage{thmtools, thm-restate}

\newtheorem{lemma}{Lemma}
\newtheorem{theorem}{Theorem}

\renewcommand{\algorithmicrequire}{\textbf{Input:}}
\renewcommand{\algorithmicensure}{\textbf{Output:}}

\usepackage{xspace}
\DeclareRobustCommand{\eg}{e.g.,\@\xspace}
\DeclareRobustCommand{\ie}{i.e.,\@\xspace}
\DeclareRobustCommand{\wrt}{w.r.t.\@\xspace}

\newcommand{\rmax}{r_{\max}}
\newcommand{\wt}[1]{\widetilde{#1}}
\newcommand{\wh}[1]{\widehat{#1}}

\newcommand{\wb}[1]{\overline{#1}}
\DeclareMathOperator*{\argmax}{\arg\,\max}
\DeclareMathOperator*{\argmin}{\arg\,\min}
\DeclareMathOperator*{\arginf}{\arg\,\inf}

\usepackage[colorinlistoftodos, textwidth=20mm]{todonotes}
\definecolor{citrine}{rgb}{0.89, 0.82, 0.04}
\definecolor{blued}{RGB}{70,197,221}

\begin{document}

\twocolumn[
\icmltitle{Importance Weighted Transfer of Samples in Reinforcement Learning}



\icmlsetsymbol{equal}{*}

\begin{icmlauthorlist}
\icmlauthor{Andrea Tirinzoni}{polimi}
\icmlauthor{Andrea Sessa}{polimi}
\icmlauthor{Matteo Pirotta}{inria}
\icmlauthor{Marcello Restelli}{polimi}
\end{icmlauthorlist}

\icmlaffiliation{polimi}{Politecnico di Milano, Milan, Italy}
\icmlaffiliation{inria}{SequeL Team, INRIA Lille, France}

\icmlcorrespondingauthor{Andrea Tirinzoni}{andrea.tirinzoni@polimi.it}

\icmlkeywords{ICML, Machine Learning, Reinforcement Learning, Transfer Learning}

\vskip 0.3in
]



\printAffiliationsAndNotice{}  

\begin{abstract}
We consider the transfer of experience samples (\ie tuples $\langle s,a,s',r \rangle$) in reinforcement learning (RL), collected from a set of source tasks to improve the learning process in a given target task.
Most of the related approaches focus on selecting the most relevant source samples for solving the target task, but then all the transferred samples are used without considering anymore the discrepancies between the task models.
In this paper, we propose a model-based technique that automatically estimates the relevance (importance weight) of each source sample for solving the target task. 
In the proposed approach, all the samples are transferred and used by a batch RL algorithm to solve the target task, but their contribution to the learning process is proportional to their importance weight.
By extending the results for importance weighting provided in supervised learning literature, we develop a finite-sample analysis of the proposed batch RL algorithm.
Furthermore, we empirically compare the proposed algorithm to state-of-the-art approaches, showing that it achieves better learning performance and is very robust to negative transfer, even when some source tasks are significantly different from the target task.
\end{abstract}

\section{Introduction}
\label{S:intro}
The goal of transfer in Reinforcement Learning (RL) \cite{sutton1998reinforcement} is to speed-up RL algorithms by reusing knowledge obtained from a set of previously learned tasks. The intuition is that the experience made by learning \textit{source tasks} might be useful for solving a related, but different, \textit{target task}. Transfer across multiple tasks may be achieved in different ways. The available approaches differ in the type of information transferred (\eg samples, value functions, parameters, policies, etc.) and in the criteria used to establish whether such knowledge could be beneficial for solving the target or not.

This work focuses on the problem of transferring samples from a set of source MDPs to augment the dataset used to learn the target MDP. To motivate our approach, consider a typical learning scenario where samples are costly to obtain. This is often the case in robotics applications, where the interaction with the real environment could be extremely time-consuming, thus reducing the number of samples available. The typical remedy of adopting a simulator often leads to sub-optimal solutions due to the differences with respect to the real environment. A more effective approach is to transfer the simulated samples to speed-up learning in the target task.

The transfer of samples has been widely studied in the supervised learning community. In particular,~\citet{crammer2008learning} formalized the problem from a theoretical perspective and provided generalization bounds for the transfer scenario. An interesting result is a trade-off between the number of tasks from which to transfer and the total number of samples.
In RL,~\citet{taylor2008transferring} and~\citet{lazaric2008transfer} proposed almost simultaneously methods to transfer single samples. While the former method focused on a model-based approach, the latter one proposed a selective approach to transfer samples into a batch RL algorithm (\eg Fitted Q-Iteration~\citep{ernst2005tree}). Furthermore,~\citet{lazaric2008transfer} considered a model-free approach to compute a similarity measure between tasks, which was used to decide which samples to transfer.
More recently,~\citet{lazaric2011transfer} analyzed the transfer of samples in batch RL from a theoretical perspective, demonstrating again the trade-off between the total number of samples and the number of tasks from which to transfer.
Finally,~\citet{laroche2017transfer} proposed a way to transfer all the samples to augment the dataset used by Fitted Q-Iteration. The limitation of this approach resides in the restrictive assumption that all the tasks are assumed to share the same transition dynamics and differ only in the reward function.
For a survey on transfer in RL, we refer the reader to \cite{taylor2009transfer,lazaric2012transfer}.

One of the main drawbacks of many previous works is that, even after a detailed selection, transferred samples are used in the target task \emph{without} accounting for the differences between the original (source) MDP and the target one, thus introducing a bias even in the asymptotic case. 
In this paper, we present a novel approach to transfer samples into a batch RL algorithm.
Unlike other works, we do not assume any particular similarity between tasks besides a shared state-action space, and we develop a new model-based methodology to automatically select the relevance (importance weight) of each sample. Existing algorithms for transferring across different state-action spaces ~\citep[\eg][]{taylor2007transfer} can be straightforwardly combined to our method.
Our approach transfers all the samples, but their impact in solving the target task is proportional to their importance weight.
To compute the importance weight of each sample, we rely on a non-parametric estimate of the MDP structure. In particular, we adopt Gaussian processes~\citep{rasmussen2006gaussian} to estimate the reward and state transition models of the source and target tasks from samples. Then, we propose a robust way to compute two sets of importance weights, one for the reward model and one for the transition model. We introduce an approximate value iteration algorithm based on Fitted Q-iteration that uses such weights to account for the distribution shift introduced by the different MDPs, thus implicitly selecting which samples have higher priority based on their likelihood to be generated from the target MDP. We provide a theoretical analysis showing the asymptotic correctness of our approach and an empirical evaluation on two classical RL domains and a real-world task.

 
\section{Preliminaries}
\label{S:preliminaries}

In this section, we start by introducing our mathematical notation. Then, we recall concepts of Markov decision processes and approximate value iteration. Finally, we formalize the transfer settings considered in this work.

\paragraph{Notation.}
For a measurable space $\langle \Omega, \sigma_\Omega \rangle$, we denote by $\Delta(\Omega)$ the set of probability measures over $\sigma_{\Omega}$ and by $\mathcal{B}(\Omega,L)$ the space of measurable functions over $\Omega$ bounded by $0 < L < \infty$,
\ie $\forall f \in \mathcal{B}(\Omega,L), \forall x,~|f(x)| \leq L$.
Given a probability measure $\mu$, we define the $\ell_p$-norm of a measurable function $f$ as $\left\lVert f \right\rVert_{p,\mu} = \left( \int |f|^p d\mu \right)^{1/p}$. 
Let $z_{1:N}$ be a $\mathcal{Z}$-valued sequence $(z_1,\ldots,z_N)$ for some
space $\mathcal{Z}$.
For $\mathcal{D}_N = z_{1:N}$, the empirical norm of a function
$f : \mathcal{Z} \to \mathbb{R}$ is
$\|f\|^p_{p,\mathcal{D}_N} \coloneqq \frac{1}{N} \sum_{i=1}^{N} |f(z_i)|^p$.
Note that when $Z_i \sim \mu$, we have that
$\mathbb{E}[\|f\|_{p,\mathcal{D}_N}^p]  = \|f\|_{p,\mu}^p $.
Whenever the subscript $p$ is dropped, we implicitly consider the $\ell_2$-norm.

\paragraph{Markov Decision Process.}
We define a discounted Markov Decision Process (MDP) as a tuple $M = \langle \mathcal{S},\mathcal{A},\mathcal{P},\mathcal{R},\gamma \rangle$, where $\mathcal{S}$ is a measurable state space, $\mathcal{A}$ is a finite set of actions, $\mathcal{P} : \mathcal{S} \times \mathcal{A} \rightarrow \Delta(\mathcal{S})$ is the transition probability kernel, $\mathcal{R}: \mathcal{S} \times \mathcal{A} \rightarrow \Delta(\mathbb{R})$ is the reward probability kernel, and $\gamma \in [0,1)$ is the discount factor. We suppose ${R}(s,a) = \mathbb{E}\left[ \mathcal{R}(\cdot | s,a) \right]$ is uniformly bounded by $\rmax$. A Markov randomized policy maps states to distributions over actions as $\pi : \mathcal{S} \to \Delta(\mathcal{A})$.
As a consequence of taking an action $a_t$ in $s_t$, the agent receives a reward $r_t \sim \mathcal{R}(\cdot|s_t,a_t)$ and the state evolves accordingly to $s_{t+1} \sim \mathcal{P}(\cdot|s_t,a_t)$.
We define the action-value function of a policy $\pi$ as $Q^\pi(s,a) = \mathbb{E}\left[ \sum_{t=0}^\infty \gamma^t r_t \mid M, \pi, s_0 = s, a_0 = a \right]$ and the optimal action-value function as $Q^*(s,a) = \sup_\pi\ Q^\pi(s,a)$ for all $(s,a)$. Notice that $Q$ is bounded by $Q_{\max} \coloneqq \frac{\rmax}{1-\gamma}$.
Then, the optimal policy $\pi^*$ is a policy that is greedy with respect to $Q^*$, \ie for all $s\in\mathcal{S}$, $\pi^*(s) \in \argmax_{a\in\mathcal{A}} \{ Q^*(s,a) \}$.
The optimal action-value function is also the unique fixed-point of the \emph{optimal Bellman operator} $L^* : \mathcal{B}(\mathcal{S} \times \mathcal{A},Q_{\max}) \rightarrow \mathcal{B}(\mathcal{S} \times \mathcal{A},Q_{\max})$, which is defined by $(L^* Q)(s,a) \coloneqq {R}(s,a) + \gamma \int_\mathcal{S} \mathcal{P}(ds'|s,a) \max_{a'}Q(s',a')$~\citep[\eg][]{puterman1994markov}. 

\paragraph{Approximate solution.}
Fitted Q-Iteration (FQI)~\citep{ernst2005tree} is a batch RL algorithm that belongs to the family of Approximate Value Iteration (AVI).
AVI is a value-based approach that represents Q-functions by a hypothesis space $\mathcal{H} \subset \mathcal{B}(\mathcal{S} \times \mathcal{A},Q_{\max})$ of limited capacity. Starting from an initial action-value function $Q_0 \in \mathcal{H}$, at each iteration $k \geq 0$, AVI approximates the application of the optimal Bellman operator in $\mathcal{H}$ such that $Q_{k+1} \approx L^*Q_k$.
Formally, let $\mathcal{D}_N = \{\langle s_i,a_i, s'_i,r_i \rangle \}_{i=1}^N$ be a set of transitions such that $(s_i, a_i) \sim \mu$ and define the \emph{empirical} optimal Bellman operator as $(\wh{L}^*Q)(s_i,a_i) \coloneqq r_i + \gamma \max_{a'} Q(s'_i,a')$. Then, at each iteration $k$, FQI computes
\begin{equation}\label{E:fqi.minimization}
        Q_{k+1} = \argmin_{h\in\mathcal{H}} \left\| h - \wh{L}^*Q_k \right\|_{\mathcal{D}_N}^2.
\end{equation}

\paragraph{Transfer settings.}
We consider a set of tasks, \ie MDPs, $\{M_j = \langle \mathcal{S}, \mathcal{A}, \mathcal{P}_j, \mathcal{R}_j \rangle, j=0,\ldots, m\}$, where $M_0$ denotes the target and $M_1, \dots ,M_m$ the sources.
We suppose all tasks share the same state-action space and have potentially different dynamics and reward.
Suppose that, for $j=0,\dots,m$, we have access to a dataset of $N_j$ samples from the $j$-th MDP, $\mathcal{D}_j = \{\langle s_i,a_i,s'_i, r_i \rangle\}_{i=1}^{N_j}$, where state-action pairs are drawn from a common distribution $\mu \in \Delta(\mathcal{S} \times \mathcal{A})$.\footnote{This assumption can be relaxed at the price of a much more complex theoretical analysis.}
The goal of transfer learning is to use the samples in $\mathcal{D}_1,\dots,\mathcal{D}_m$ to speed up the learning process in the target task $M_0$. 

%
%

\section{Importance Weights for Transfer}
In this section, we introduce our approach to the transfer of samples.
Recall that our goal is to exploit at best samples in $\{\mathcal{D}_1,\dots,\mathcal{D}_M\}$ to \emph{augment} the dataset $\mathcal{D}_0$ used by FQI to solve the target task, thus speeding up the learning process.
In the rest of the paper we exploit the fact that FQI decomposes the RL problem into a sequence of supervised learning problems. It is easy to notice that the optimization problem~\eqref{E:fqi.minimization} is an instance of empirical risk minimization, where $X_i = (s_i, a_i)$ are the input data, $Y_i = (\wh{L}^*Q_k)(s_i,a_i)$ are the targets, and $\mathcal{L}(f(X_i), Y_i) = |f(X_i) - Y_i|^2$ is a squared loss.

As mentioned in the introduction, we aim to exploit all the available samples to solve the target task. 
Suppose we adopt a naive approach where we concatenate all the samples, \ie $\wt{\mathcal{D}} = \bigcup_{j=0}^{m} \mathcal{D}_j = \bigcup_{j=0}^{m} \bigcup_{i=0}^{N_j} \langle s^j_i, a_i^j, {s'}^j_i, r_i^j \rangle$, to solve~\eqref{E:fqi.minimization}.
This approach suffers from sample selection bias~\citep{cortes2008sample}, \ie samples are collected from different distributions or domains.
In fact, although we assumed state-action pairs to be sampled from a fixed task-independent distribution, the target variables $Y$ are distributed according to the MDP they come from.

A standard technique used to correct the bias or discrepancy induced by the distribution shift is \emph{importance weighting}.
This technique consists in weighting the loss function to emphasize the error on some samples and decrease it on others, to correct the mismatch between distributions~\citep{cortes2008sample}. 
The definition of the importance weight for the point $X$ is $w(X) = P(X) / Q(X)$ where $P$ is the distribution of the target, and $Q$ is the distribution according to which sample $X$ is collected.
In our specific case, given an arbitrary sample $(X,Y)$, its joint distribution under MDP $M_j$ is $\mathbb{P}((X,Y)|M_j) = \mathbb{P}(Y| X, M_j) \mu(X)$. 
Denote by $(X_i^{(j)}, Y_i^{(j)})$ the $i$-th sample drawn from MDP $M_j$, then its importance weight is given by $w(X_i^{(j)}, Y_i^{(j)}) = \frac{\mathbb{P}(Y_i^{(j)}|X_i^{(j)},M_0)}{\mathbb{P}(Y_i^{(j)}|X_i^{(j)},M_j)}$.

\begin{algorithm}[t]
\caption{Importance Weighted Fitted Q-Iteration}\label{alg:wfqi}
\begin{algorithmic}
        \REQUIRE The number of iterations $K$, a dataset $\wt{\mathcal{D}}^+ = \bigcup_{j=0}^m \bigcup_{i=0}^{N_j} \left\{ s_{i}^{(j)},a_{i}^{(j)},{s'}_{i}^{(j)},r_{i}^{(j)},\widetilde{w}_{r,i}^{(j)},\widetilde{w}_{p,i}^{(j)} \right\}$, a hypothesis space $\mathcal{H}$
\ENSURE Greedy policy $\pi_{K}$
\STATE $\widehat{R} \leftarrow \arginf_{h\in\mathcal{H}}\ \frac{1}{Z_r} \sum_{j,i} \widetilde{w}_{r,i}^{(j)} \left| h(s_i^{(j)}, a_i^{(j)}) - r_{i}^{(j)}\right|^2$
\STATE $Q_0 \leftarrow \widehat{R}$
\FOR{$k=0,\dots,K-1$}
\STATE $Y_i^{(j)} \leftarrow \wt{L}^*Q_k(s_i^{(j)}, a_i^{(j)}), \quad \forall i,j$
\STATE $Q_{k+1} \leftarrow \arginf\limits_{h \in \mathcal{H}} \frac{1}{Z_p} \sum_{j,i} \wt{w}_{i,p}^{(j)} \left| h(s_i^{(j)}, a_i^{(j)}) - Y_i^{(j)} \right|^2$
\ENDFOR
\STATE $\pi_{K}(s) \leftarrow \argmax_{a \in \mathcal{A}} \{ Q_{K}(s,a)\}, \;\; \forall s \in \mathcal{S}$
\end{algorithmic}
\end{algorithm}

By feeding FQI on the full dataset $\wt{\mathcal{D}}$ with samples weighted by $w_{i}^{(j)}$ (for short), we get an algorithm that automatically selects which samples to exploit, \ie those that, based on the importance weights, are more likely to be generated from the target MDP. 
This approach looks appealing but presents several issues. First, the distribution $\mathbb{P}(Y|X,M_j)$ is, even in the case where the MDPs are known, very hard to characterize. 
Second, consider a simple case where we have a source MDP with the same transition dynamics as the target, but with entirely different reward. 
Then, the importance weights defined above are likely to be very close to zero for any source sample, thus making transfer useless.
However, we would like a method able to leverage the fact that transition dynamics do not change, thus transferring only that part of the sample.

To overcome the second limitation, we consider the following variation of the FQI algorithm. At the first iteration of FQI, we use all the samples to fit a model $\wh{R}\approx {R}$ of the target reward function:
\begin{equation}\label{eq:rhat}
        \wh{R} = \arginf_{h \in \mathcal{H}} \frac{1}{Z_r} \sum_{j=0}^{m} \sum_{i=0}^{N_j} w_{r,i}^{(j)} \left| h(X_i^{(j)}) - r_{i}^{(j)} \right|^2,
\end{equation}
where $\mathcal{H} \subset \mathcal{B}(\mathcal{S} \times \mathcal{A},Q_{max})$ is the hypothesis space we consider to represent action-value functions\footnote{Differently from other works ~\citep[\eg][]{farahmand2012value,tosatto2017boosted}, we suppose, for the sake of simplicity, the hypothesis space to be bounded by $Q_{\max}$. Although this is a strong assumption, it can be relaxed by considering truncated functions. We refer the reader to \cite{gyorfi2006distribution} for the theoretical consequences of such relaxation.} and
\begin{equation} \label{eq:wr-ideal}
        w_{r,i}^{(j)} = \frac{\mathcal{R}_0(r_{i}^{(j)} | X_{i}^{(j)})}{\mathcal{R}_j(r_{i}^{(j)} | X_i^{(j)})}.
\end{equation} 
Problem~\eqref{eq:rhat} is unbiased if $Z_r = \sum_{j=0}^m N_j$, though $Z_r = \sum_{i,j} w_{r,i}^{(j)}$ is frequently used since it provides lower variance. The theoretical analysis is not affected by the choice of $Z_r$, while in the experiments we will use $Z_r = \sum_{i,j} w_{r,i}^{(j)}$.
Then, at each iteration $k\geq 0$, FQI updates the Q-function as:
\begin{equation}\label{eq:qk}
        Q_{k+1} = \arginf_{h \in \mathcal{H}} \frac{1}{Z_p} \sum_{j=0}^{m} \sum_{i=0}^{N_j} w_{p,i}^{(j)} \left| h(X_i^{(j)}) - \wt{Y}_{i}^{(j)} \right|^2
\end{equation}
where $\wt{Y}_{i}^{(j)} = \wt{L}^*Q_k(X_i^{(j)}) \coloneqq \wh{R}(s_i^{(j)}, a_i^{(j)}) + \gamma \max_{a'} Q_{k}({s'}_i^{(j)}, a')$ and $Q_0 = \wh{R}$.
Intuitively, instead of considering the reward $r_i^{(j)}$ in the dataset $\wt{\mathcal{D}}$, we use $\wh{R}(s_i^{(j)}, a_i^{(j)})$.
Since the stochasticity due to the reward samples is now removed, only the transition kernel plays a role, and the importance weights are given by:
\begin{equation} \label{eq:wp-ideal}
        w_{p,i}^{(j)} = \frac{\mathcal{P}_0({s'}_{i}^{(j)} | X_{i}^{(j)})}{\mathcal{P}_j({s'}_{i}^{(j)} | X_{i}^{(j)})}.
\end{equation}
The resulting algorithm, named Importance Weighted Fitted Q-Iteration (IWFQI), is shown in Algorithm~\ref{alg:wfqi}. 
In practice, we have to compute an estimate of $w_{r,i}^{(j)}$ and $w_{p,i}^{(j)}$ since $\mathcal{P}_j$ and $\mathcal{R}_j$ are unknown quantities.
We postpone this topic to Section~\ref{S:iw.estimation} since several approaches can be exploited. Instead, in the following section, we present a theoretical analysis that is independent of the way the importance weights are estimated.

\section{Theoretical Analysis}\label{S:theory}
We now study the theoretical properties of our IWFQI algorithm.
We analyze the case where we have samples from one source task, but no samples from the target task are available, \ie $m=1$, $N_0=0$, and $N_1 = N$.
A generalization to the case where target samples or samples from more sources are available is straightforward, and it only complicates our derivation.
To ease our notation, we adopt the subscript ``T'' and ``S'' to denote the target and the source.
Furthermore, we want to emphasize that the results provided in this section are \emph{independent from the way the importance weights are estimated}.

Consider the sequence of action-value functions $Q_0,Q_1,\ldots,Q_K$ computed by IWFQI.
At each iteration $k$, we incur in an error $\epsilon_k = L^*Q_k - Q_{k+1}$ in approximating the optimal Bellman operator. 
Our goal is to bound, in terms of such errors, $\|Q^*-Q^{\pi_k}\|_{1,\rho}$, \ie the expected error under distribution $\rho$ between the performance of the optimal policy and that of the policy $\pi_k$ greedy \wrt $Q_k$. Here $\rho$ is an arbitrary evaluation distribution over $\mathcal{S} \times \mathcal{A}$ that the user can freely choose. In practice, it might coincide with the sampling distribution $\mu$. 
Since IWFQI belongs to the family of AVI algorithms, we can resort to Theorem 3.4 in~\citep{farahmand2011regularization}. We report here the version with $\ell_1$-norm for the sake of completeness.
\begin{theorem} \label{th:err-prop}
        (Theorem 3.4 of \cite{farahmand2011regularization}) Let $K$ be a positive integer and $Q_{\max}\leq\frac{\rmax}{1-\gamma}$. Then, for any sequence $(Q_k)_{k=0}^K \subset B(\mathcal{S}\times\mathcal{A},Q_{\max})$ and the corresponding sequence $(\epsilon_k)_{k=0}^K$, where $\epsilon_k = L^* Q_k - Q_{k+1}$, we have:
\begin{align*}
\|Q^* &- Q^{\pi_K}\|_{1,\rho} \leq \frac{2\gamma}{(1-\gamma)^2}
\Bigg[ 2\gamma^K Q_{\max}\\
& + \inf_{b \in [0,1]}\left\{ C_{\textsc{VI},\rho,\mu}^{\frac{1}{2}}(K;b)\mathcal{E}^{\frac{1}{2}}(\epsilon_0,\ldots,\epsilon_{K-1};b) \right\} \Bigg],
\end{align*}
where:
\begin{equation*}
\mathcal{E}(\epsilon_0,\ldots,\epsilon_{K-1};b) = \sum_{k=0}^{K-1} \alpha_k^{2b}\|\epsilon_k\|_\mu^2.
\end{equation*}
\end{theorem}
We refer the reader to Chapter 3 of \cite{farahmand2011regularization} for the definitions of the coefficients $C_{\textsc{VI},\rho,\mu}$ and $\alpha_k$.

Intuitively, the bound given in Theorem~\ref{th:err-prop} depends on the errors made by IWFQI in approximating the optimal Bellman operator at each iteration. Thus, our problem reduces to bounding such errors. \citet{cortes2010learning} already provided a theoretical analysis of importance weighted regression. However, their results are not immediately applicable to our case since they only consider a regression problem where the target variable $Y$ is a deterministic function of the input $X$. On the other hand, we have the more general regression estimation problem where $Y$ is a random variable, and we want to learn its conditional expectation given $X$.
Thus, we extend Theorem 4 of~\cite{cortes2010learning} to provide a bound on the expected $\ell_2$-error $\|\wh{h}-h^*\|_{\mu}$ between the hypothesis $\wh{h}$ returned by a weighted regressor (with estimated weights $\widetilde{w}$) and the regression function $h^*$.
Following~\citep{cortes2010learning}, we denote by $Pdim(U)$ the pseudo-dimension of a real-valued function class $U$.
The proof is in the appendix.
\begin{restatable}{theorem}{thmthree} \label{th:regr}
Let $\mathcal{H} \subset \mathcal{B}(X,F_{\max})$ be a functional space. Suppose we have a dataset of $N$ i.i.d. samples $ \mathcal{D} = \{(x_i,y_i)\}$ distributed according to $Q(X,Y) = q(Y | X)\mu(X)$, while $P(X,Y) = p(Y | X)\mu(X)$ is the target distribution. Assume $|Y| \leq F_{\max}$ almost surely. Let $w(x,y) = \frac{p(y | x)}{q(y | x)}$, $\widetilde{w}(x,y)$ be any positive function,
$\wh{h}(x) = \argmin_{f \in \mathcal{H}} \wh{\mathbb{E}}_{\mathcal{D}}\left[ \wt{w}(X,Y) |f(X)-Y|^2 \right]$, 
$h^*(x) = \mathbb{E}_p[Y|x]$, $g(x) = \mathbb{E}_q[\wt{w}(x,Y) | x] - 1$, and $M(\wt{w}) = \sqrt{\mathbb{E}_{Q}[\widetilde{w}(X,Y)^2]} + \sqrt{\wh{\mathbb{E}}_{\mathcal{D}}[\wt{w}(X,Y)^2]}$, where $\widehat{\mathbb{E}}_{\mathcal{D}}$ denotes the empirical expectation on $\mathcal{D}$. Furthermore, assume $d = Pdim(\{ |f(x)-y|^2\ : f\in\mathcal{H} \}) < \infty$ and $\mathbb{E}_Q[\widetilde{w}(X,Y)^2] < \infty$. Then, for any $\delta>0$, the following holds with probability at least $1-2\delta$:
\begin{equation*}
\begin{split}
        \|\wh{h}-h^*\|_\mu &\leq
\inf_{f \in \mathcal{H}}\ \|f-h^*\|_\mu +
F_{\max}\sqrt{\|g\|_{1,\mu}} \\
&+ 2^{13/8}F_{\max} \sqrt{M(\wt{w})} 
\left(
\frac{d\log\frac{2Ne}{d} + \log\frac{4}{\delta}}{N}
\right)^{\frac{3}{16}}\\
&+ 2F_{\max}\|\widetilde{w}-w\|_Q
\end{split}
\end{equation*}
\end{restatable}
Notice that this result is of practical interest outside of the reinforcement learning field.
Here it is used to bound the errors $\|\epsilon_k\|_\mu$ in order to state the following result.
\begin{restatable}{theorem}{thmfour} \label{th:wfqi-err}
        Let $\mathcal{H} \subset B(\mathcal{S} \times \mathcal{A},Q_{\max})$ be a hypothesis space, $\mu$ a distribution over $\mathcal{S}\times\mathcal{A}$, $(Q_{i})_{i=0}^{k+1}$ a sequence of Q-functions as defined in Equation~\eqref{eq:qk}, and $L^*$ the optimal Bellman operator of the target task. Suppose to have a dataset of N i.i.d. samples $\mathcal{D}$ drawn from the source task $M_S$ according to a joint distribution $\phi_S$. Let $w_p$, $w_r$ denote the ideal importance weights defined in \eqref{eq:wp-ideal} and \eqref{eq:wr-ideal}, and $\wt{w}_r(r|s,a)$, $\wt{w}_p(s'|s,a)$ denote arbitrary positive functions with bounded second moments.
        Define $g_r(s,a) = \mathbb{E}_{\mathcal{R}_S}[\wt{w}_r(r|s,a)|s,a]-1$, $M(\wt{w}_r)=\sqrt{\mathbb{E}_{\phi^R_S}[\wt{w}_r(r|s,a)^2]} + \sqrt{\wh{\mathbb{E}}_{\mathcal{D}}[\wt{w}_r(r|s,a)^2]}$, where $\phi^R_S(r|s,a) = \mu(s,a)\mathcal{R}_S(r|s,a)$.
        Similarly, define $g_p$, $M(\wt{w}_p)$, and $\phi^P_S(s'|s,a)$ for the transition model. Then, for any $\delta>0$, with probability at least $1-4\delta$:
{\small
\begin{align*}
        &\|L^* Q_k - Q_{k+1}\|_\mu
\leq Q_{\max}\sqrt{\|g_p\|_{1,\mu}} + 2\rmax \sqrt{\|g_r\|_{1,\mu}} \\
&+ 2Q_{\max}\|\wt{w}_p-w_p\|_{\phi^P_S} + 4\rmax\|\wt{w}_r-w_r\|_{\phi^R_S} \\
&+ \inf_{f \in \mathcal{H}} \|f-(L^*)^{k+1}Q_0\|_{\mu} + 2 \inf_{f \in \mathcal{H}} \|f-R\|_{\mu} \\ 
&+ \frac{Q_{\max}}{2^{-\frac{13}{8}}}\left(\sqrt{M(\wt{w}_p)}+2\sqrt{M(\wt{w}_r)}\right)
\left(\frac{d\log\frac{2Ne}{d} + \log\frac{4}{\delta}}{N} \right)^{\frac{3}{16}} \\ 
&+ \sum_{i=0}^{k-1} (\gamma C_{\textsc{AE}}(\mu))^{i+1}\|\epsilon_{k-i-1}\|_\mu,
\end{align*}}
where $C_{\textsc{AE}}$ is the concentrability coefficient of one-step transitions as defined in~\citep[][Definition 5.2]{farahmand2011regularization}.
\end{restatable}
As expected, four primary sources of error contribute to our bound: (i) the bias due to estimated weights (first four terms), (ii) the approximation error (fifth and sixth term), (iii) the estimation error (seventh term), (iv) the propagation error (eighth term). Notice that, assuming to have a consistent estimator for the importance weights (an example is given in Section \ref{S:iw.estimation}), the bias term vanishes as the number of samples $N$ tends to infinity. 
Furthermore, the estimation error decreases with $N$, thus vanishing as the number of samples increases. Thus, in the asymptotic case our bound shows that the only source of error is due to the limited capacity of the functional space $\mathcal{H}$ under consideration, as in most AVI algorithms. Furthermore, we notice that fitting the reward function and using it instead of the available samples propagates an error term through iterations, \ie the approximation error $\inf_{f\in\mathcal{H}}\|f-R\|_\mu$. If we were able to estimate the importance weights for the typical case where both reward and transition samples are used, we could get rid of such error. However, since the resulting weights somehow depend on the joint densities between $\mathcal{P}$ and $\mathcal{R}$, we expect their variance, as measured by $M(\widetilde{w})$, to be much bigger, thus making the resulting bound even larger. Furthermore, we argue that, when the reward function is simple enough and only a limited number of samples is available, a separate fit might be beneficial even for plain FQI. In fact, the variance of the empirical optimal Bellman operator can be reduced by removing the source of stochasticity due to the reward samples at the cost of propagating a small approximation error through iterations.
The bounds for AVI,~\citep[\eg][]{munos2008finite,farahmand2011regularization,farahmand2012value}, can be straightforwardly extended to such case by adopting a procedure similar to the one described in the proof of Theorem~\ref{th:wfqi-err}. Finally, in most practical applications the reward function is actually known and, thus, does not need to be fitted. In such cases, it is possible to get rid of the corresponding terms in Theorem \ref{th:wfqi-err}, allowing transfer to occur without errors even when rewards are completely different between tasks.

\section{Estimation of Importance Weights}\label{S:iw.estimation}
In this section, we specify how to compute the importance weights. Since $\mathcal{P}$ and $\mathcal{R}$ are unknown, we only have access to an estimation of $w_{r,i}^{(j)}$ and $w_{p,i}^{(j)}$ used in~\eqref{eq:wr-ideal} and~\eqref{eq:wp-ideal}, respectively.
To obtain an approximation of the unknown densities, we consider Gaussian Processes (GPs) although any distribution matching technique and/or probabilistic model can be used.

\paragraph{Gaussian Processes.} We use the available samples to fit two Gaussian processes (GPs)~\cite{rasmussen2006gaussian} for each task $M_j$: one for the transition model $\mathcal{P}_j$ and one for the reward model $\mathcal{R}_j$.
To motivate our choice, GPs have been successfully adopted to model stochastic dynamical systems with high-dimensional and continuous state-action spaces in many existing works~\citep[\eg][]{rasmussen2003gaussian,deisenroth2011pilco,doshi2016hidden,Berkenkamp2017safemb}.
For the sake of simplicity, we only show how to compute the importance weights for the reward model. Our procedure straightforwardly generalizes to the transition model.

Given a sample $\langle s,a,r \rangle$ from the $j$-th task, the $j$-th GP returns a Gaussian distribution over the reward's mean, \ie $\wb{r}(s,a) \sim \mathcal{N}(\mu_{GP_j}(s,a),\sigma_{GP_j}^2(s,a))$, which, together with the target GP's prediction, induces a distribution over the importance weights. In practice, the choice of a single importance weight can rely on some statistics of such distribution (e.g., its mean or mode). Perhaps not surprisingly, this is made non-trivial by the fact that explicitly characterizing such distribution is very complicated, and computing empirical statistics requires an expensive repeated sampling from the GPs' posteriors. Interestingly, the following theorem shows that this is not necessary when the reward model follows a Gaussian law, as the expected weights under their unknown distribution can be computed in closed-form.

\begin{restatable}[Reward Weights in Gaussian Models]{theorem}{thmone}\label{th:gauss-w}
Assume each task to have Gaussian reward distribution $\mathcal{R}_j(\cdot |s,a) = \mathcal{N}\left(\mu_r^{(j)}(s,a),\sigma^2_{j}(s,a)\right)$ with unknown mean.
Given the available samples in $\wt{\mathcal{D}}$, we build an estimate of the reward distribution such that, for any MDP $M_j$, $\wb{r}^{(j)}(s,a) \sim \mathcal{N}(\mu_{GP_j}(s,a),\sigma_{GP_j}^2(s,a))$.
Then, given a sample $\langle s,a,r \rangle$ from the $j$-th MDP, its importance weight $w = \frac{\mathcal{N}\left(r | \wb{r}^{(0)}(s,a), \sigma_0^2(s,a)\right)}{\mathcal{N}\left(r | \wb{r}^{(j)}(s,a), \sigma_j^2(s,a)\right)} \sim \mathcal{G}$, where $\mathcal{G}$ is the distribution induced by the GPs' predictions.
Let $C = \frac{\sigma_j^2(s,a)}{\sigma_j^2(s,a)-\sigma_{GP_j}^2(s,a)}$ and suppose  $\sigma_{GP_j}^2(s,a) < \sigma_j^2(s,a)$, then
\begin{equation}\label{eq:exp-weights}
\resizebox{7.45cm}{!}{$
\mathbb{E}_{\mathcal{G}}\left[w\right] =
C\ \frac{\mathcal{N}\left(r \big| \mu_{GP_0}(s,a),\sigma_0^2(s,a) + \sigma_{GP_0}^2(s,a)\right)}{\mathcal{N}\left(r \big| \mu_{GP_j}(s,a),\sigma_j^2(s,a) - \sigma_{GP_j}^2(s,a)\right)}.
$}
\end{equation}
\end{restatable}
The proof is in Appendix~\ref{A:gaussianmodels.rewardweights}.
In practice, we estimate the importance weights by taking their expectation as in~\eqref{eq:exp-weights}, \ie $\widetilde{w} = \mathbb{E}_{\mathcal{G}}[w]$.
Intuitively, using the expected weights is more robust than merely taking the ratio of the estimated densities.
Furthermore, the estimated weights converge to the true ones when the GP predictions are perfect, \ie when $\mu_{GP}(s,a) = \mu_r(s,a)$ and $\sigma_{GP}^2(s,a) \to 0$, both in the source and in the target.
This is a significant advantage over the more common approach of density-ratio estimation \cite{sugiyama2012density}, where a parametric form for the weight function is typically assumed. 
One drawback is that the expectation diverges when $\sigma_{GP_j}^2(s,a) > \sigma_j^2$, that is, when the source GP has a prediction variance that is greater than the intrinsic noise of the model.
Notice, however, that this happens very rarely since the source GP is never asked to predict samples it has not seen during training. 
Furthermore, since in practice the model noise is unknown and has to be estimated, an overestimation is beneficial as it introduces a regularization effect~\cite{mohammadi2016analytic}, thus avoiding the problem mentioned above.

\section{Related Work}
In~\cite{taylor2008transferring}, the authors propose a method to transfer samples for model-based RL. Although they assume tasks might have different state-action space, inter-task mappings are used to map source samples to the target. However, the proposed method does not account for differences in the transition or reward models, which could lead to significant negative transfer when the tasks are different. Our approach, on the other hand, can selectively discard samples based on the estimated difference between the MDPs. \citet{lazaric2008transfer} compute a compliance measure between the target and source tasks and use it to specify from which tasks the transfer is more likely to be beneficial. Furthermore, a relevance measure is computed within each task to determine what are the best samples to transfer. These two measures are then combined to transfer samples into FQI. Once again, our approach does not require any explicit condition to decide what to transfer, nor does it require any assumption of similarity between the tasks. Furthermore, the compliance and relevance measures computed in~\citep{lazaric2008transfer} jointly account for both the reward and transition models, thus discarding samples when either one of the models is very different between the tasks. On the other hand, our approach can retain at least the part of the sample that is similar, at the cost of introducing a small bias. In~\cite{laroche2017transfer}, the authors propose a technique for transferring samples into FQI under the assumption that the transition dynamics do not change between the tasks. Similarly to our method, they learn the reward function at the first iteration and substitute the predicted values to the reward samples in the dataset. This allows them to safely adopt the full set of target and source samples in the remaining FQI iterations, as all tasks share the same transition model and, thus, samples are unbiased. However, we argue that this assumption of shared dynamics indeed limits the applicability of the transfer method to most real-world tasks.

In the supervised learning literature,~\citet{crammer2008learning} analyzed the transfer problem from a theoretical perspective and extended the classical generalization bounds to the case where samples are directly transferred from a set of source tasks. The most relevant result in their bounds is a trade-off between the total number of samples transferred and the total number of tasks from which transfer occurs. Increasing the first term decreases the variance, while it is likely to increase the bias due to the differences between the tasks. On the other hand, decreasing the first term also decreases the bias, but it is likely to increase the variance due to the limited number of samples. We observe that such trade-off does not arise in our case. Our method transfers all samples while accounting for the differences between the tasks. The only bias term is due to the errors in the estimation of the task models, which is likely to decrease as the number of samples increases.

Another work from the supervised learning literature that is related to our approach is~\citep{garcke2014importance}. The authors proposed a method to transfer samples from a different dataset and used importance weighting to correct the distribution shift. However, they leveraged ideas from density-ratio estimation~\citep[\eg][]{sugiyama2012density} and supposed the weight function to have a given parametric form, thus directly estimating it from the data. Conversely, we estimate the densities involved and try to characterize the weight distribution, taking its expectation as our final estimate.

\section{Experiments}
We evaluate IWFQI on three different domains with increasing level of complexity.
In all experiments, we compare our method to two existing algorithms for transferring samples into FQI: the relevance-based transfer (RBT) algorithm of \cite{lazaric2008transfer} and the shared-dynamics transfer (SDT) algorithm of \cite{laroche2017transfer}. 

\begin{figure*}[t]
\centering
\begin{minipage}[t]{0.49\linewidth}
\centering
\includegraphics[height=3.3cm]{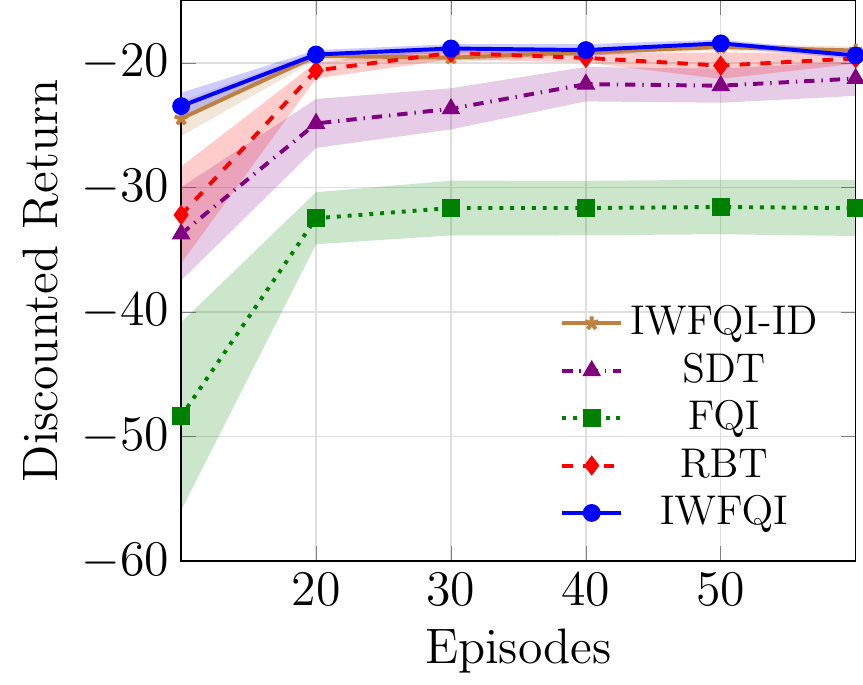}
\includegraphics[height=3.3cm]{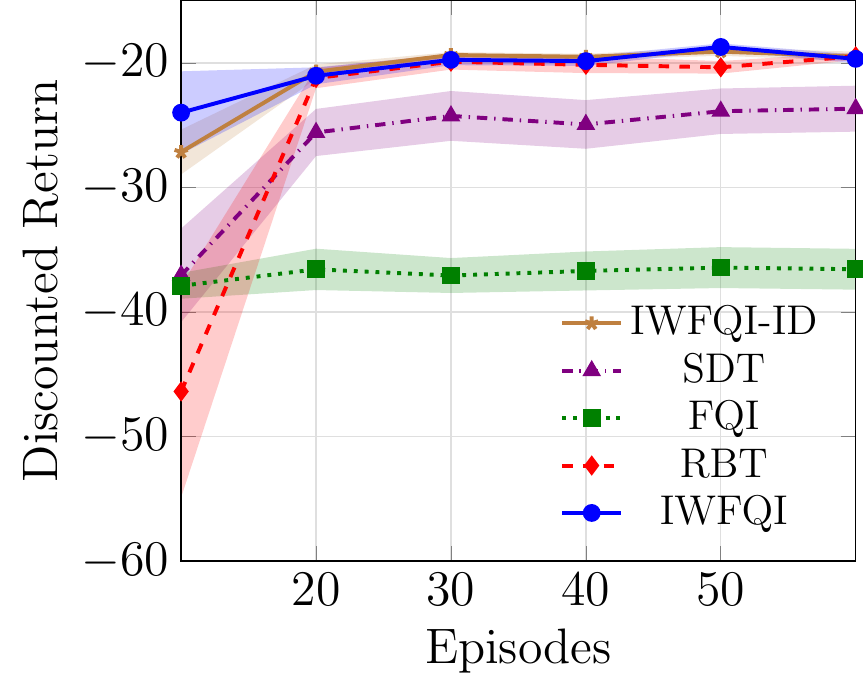}
\caption{Puddle world with $20\times 3$ episodes transferred from $3$ source tasks in the case of shared dynamics  \textit{(left)} and  puddle-based dynamics \textit{(right)}.}
    \label{fig:puddleworld}
\end{minipage}\hfill
\begin{minipage}[t]{0.49\linewidth}
\centering
\includegraphics[height=3.3cm]{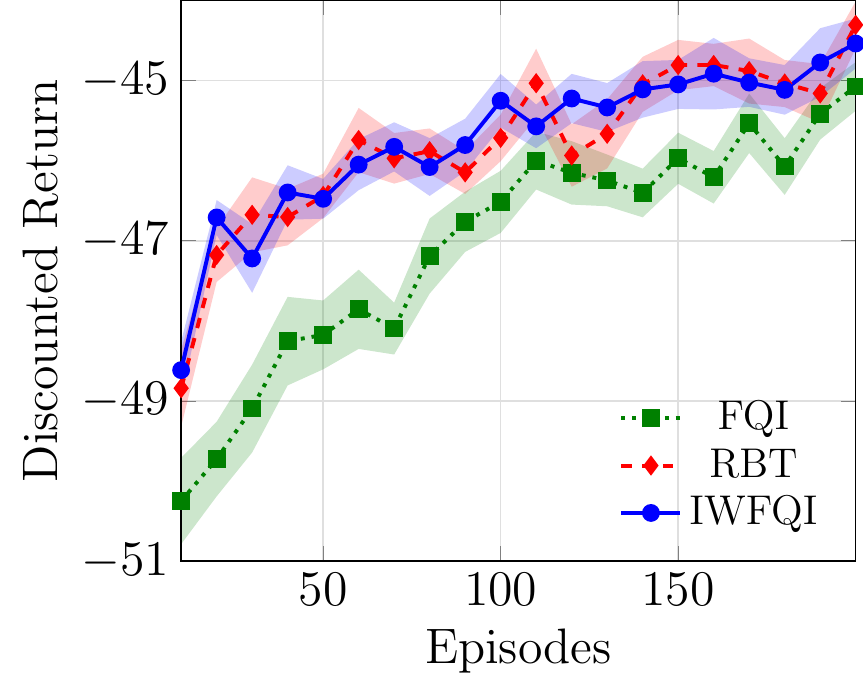} \hfill
\includegraphics[height=3.3cm]{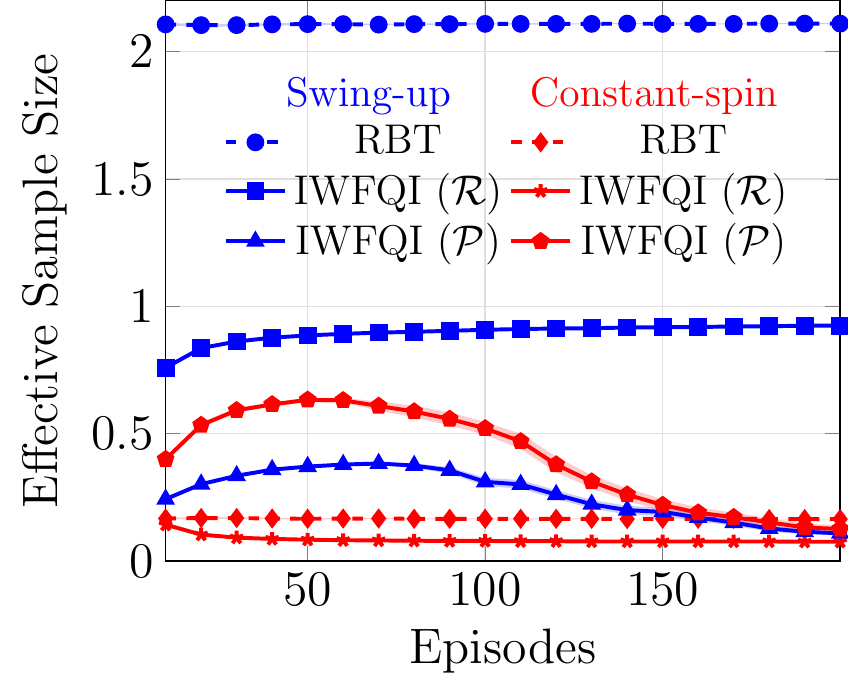}
\caption{Acrobot swing-up with $(100+50)$ episodes transferred from $2$ source tasks. \textit{(left)} learning performance. \textit{(right)} relative number of samples transferred from each source task.
    }
    \label{fig:acrobot}
\end{minipage}
\end{figure*}

\subsection{Puddle World}
Our first experimental domain is a modified version of the puddle world environment presented in~\cite{sutton1996generalization}.
Puddle world is a  discrete-action, continuous-state (stochastic) navigation problem (see Appendix~\ref{app:puddleworld} for a complete description).
At each time-step, the agent receives a reward of $-1$ plus a penalization proportional to the distance from all puddles.
Each action moves the agent by $\alpha$ in the corresponding direction. In particular, we consider two versions of the environment: (i) \emph{shared dynamics}, where $\alpha = 1$ is fixed, and (ii) \emph{puddle-based dynamics}, where $\alpha$ slows-down the agent proportionally to the distance from all puddles.

We consider three source tasks and one target task, where each task has different puddles in different locations (see Appendix~\ref{app:puddleworld}). 
For each source task, we generate a dataset of $20$ episodes from a nearly-optimal policy. We run IWFQI with weights computed according to Equation \eqref{eq:exp-weights}, where we set the model noise to be ten times the true value. For evaluating our weight estimation procedure, we also run IWFQI with ideal importance weights (computed as the ratio of the true distributions). In each algorithm, FQI is run for 50 iterations with Extra-Trees \cite{ernst2005tree}. An $\epsilon$-greedy policy ($\epsilon=0.3$) is used to collect data in the target task. 

\textbf{Shared dynamics.} We start by showing the results for $\alpha = 1$ in Figure~\ref{fig:puddleworld}\textit{(left)}.
All results are averaged over $20$ runs and are reported with $95\%$ confidence intervals. As expected, FQI alone is not able to learn the target task in such a small number of episodes. On the other hand, IWFQI has a good jump-start and converges to an optimal policy in only $20$ episodes. Interestingly, IWFQI with ideal weights has almost the same performance, thus showing the robustness of our weight estimation procedure. RBT also learns the optimal policy rather quickly. However, the limited number of target and source samples available in this experiment makes it perform significantly worse in the first episodes. Since in this version of the puddle world the dynamics do not change between tasks, SDT also achieves good performance, converging to a nearly-optimal policy.

\textbf{Puddle-based dynamics.} We also show the results for the more challenging version of the environment were puddles both penalize and slow-down the agent 
(see Figure~\ref{fig:puddleworld}\textit{(right)}). 
Notice that, in this case, transition dynamics change between tasks, thus making the transfer more challenging. Similarly, as before, our approach quickly learns the optimal policy and is not affected by the estimated weights. Furthermore, the benefits of over-estimating the model noise can be observed from the small improvement over IWFQI-ID. RBT is also able to learn the optimal policy. However, the consequences of inaccurately computing compliance and relevance are more evident in this case, where the algorithm negatively transfers samples in the first episodes. Finally, SDT still shows an improvement over plain FQI, but it is not able to learn the optimal policy due to the bias introduced by the different dynamics.

\subsection{Acrobot}
Acrobot~\cite{sutton1998reinforcement} is a classic control problem where the goal is to swing-up a two-link pendulum by applying positive or negative torque to the joint between the two links. Due to its non-linear and complex dynamics, Acrobot represents a very challenging problem, requiring a considerable amount of samples to be solved. In this experiment, we consider a multi-task scenario where robots might have different link lengths $(l_1,l_2)$ and masses $(m_1,m_2)$. Our target task is the classic Acrobot \textit{swing-up} problem, where the robot has lengths $(1.0,1.0)$ and masses $(1.0,1.0)$. Furthermore, we consider two source tasks. The first is another swing-up task where the robot has lengths $(1.1,0.7)$ and masses $(0.9,0.6)$.
The second is a \textit{constant-spin} task, where the goal is to make the first joint rotate at a fixed constant speed, with lengths $(0.95,0.95)$ and masses $(0.95,1.0)$. The exact definition of the tasks' dynamics and rewards is in Appendix~\ref{app:acrobot}.
Notice the intrinsic difficulty of transfer: the first source task has the same reward as the target but very different dynamics, and conversely for the second source task. Using nearly-optimal policies, we generate $100$ episodes from the first source and $50$ episodes from the second. We run all algorithms (except SDT since the problem violates the shared-dynamics assumption) for $200$ episodes and average over $20$ runs. Results are shown in 
Figure~\ref{fig:acrobot}\textit{(left)}.
We notice that both our approach and RBT achieve a good jump-start and learn faster than plain FQI. However, to better investigate how samples are transferred, we show the transfer ratio from each source task in Figure~\ref{fig:acrobot}\textit{(right)}. 
Since RBT transfers rewards and transitions jointly, it decides to compensate the highly biased reward samples from the constant-spin task by over-sampling the first source task. However, it inevitably introduces bias from the different dynamics. Our approach, on the other hand, correctly transfers almost all reward samples from the swing-up task, while discarding those from the constant-spin task. Due to transition noise over-estimation, IWFQI achieves an interesting adaptive behaviour: during the initial episodes, when few target samples are available, and the GPs are inaccurate, more samples are transferred. This causes a reduction of the variance in the first phases of learning that is much greater than the increase of bias. However, as more target samples are available, the transfer becomes useless, and our approach correctly decides to discard most transition samples, thus minimizing both bias and variance.

\subsection{Water Reservoir Control}

In this experiment, we consider a real-world problem where the goal is to learn how to optimally control a water reservoir system. More specifically, the objective is to learn a per-day water release policy that meets a given demand while keeping the water level below a flooding threshold. \citet{castelletti2010tree} successfully addressed such problem by adopting batch RL techniques. However, the authors proved that, due to the highly non-linear and noisy environment, an enormous amount of historical data is needed to achieve good performance. Consider now the case where a new water reservoir, for which no historical data is available, needs to be controlled. Since each sample corresponds to one day of release, learning by direct interaction with the environment is not practical and leads to poor control policies during the initial years, when only a little experience has been collected. Although we do not know the new environment, it is reasonable to assume that we have access to operational data from existing reservoirs. Then, our solution is to transfer samples to immediately achieve good performance. However, such reservoirs might be located in very different environments and weight objectives differently, thus making transfer very challenging. 

We adopt a system model similar to the one proposed in \cite{castelletti2010tree}. The state variables are the current water storage $s_t$ and day $t \in [1,365]$, while there are $8$ discrete actions, each corresponding to a particular release decision. The system evolves according to the simple mass balance equation $s_{t+1}=s_t + i_t-a_t$, where $i_t$ is the net inflow at day $t$ and is modeled as periodic function, with period of one year, plus Gaussian noise. Given the demand $d$ and the flooding threshold $f$, the reward function is a convex combination of the two objectives, $R(s_t,a_t) = -\alpha\ max\{0,s_t-f\} - \beta (\max\{0,d-a_t\})^2$, where $\alpha, \beta \geq 0$. Different tasks have different inflow functions and different reward weights, which model different geographic regions and objectives, respectively.

We collected $10800$ samples, corresponding to $30$ years of historical data, from each of $6$ source water reservoirs under a hand-coded expert policy. Further details about the tasks are given in Appendix~\ref{app:dam}. We compared our approach to FQI and RBT over the first $10$ years of learning. An $\epsilon$-greedy policy ($\epsilon=0.3$) was used to collect batches of $1$ year of samples, except for the first batch, for which an expert's policy was used. Results, averaged over $20$ runs, are shown in Figure~\ref{fig:dam}. We notice that IWFQI immediately outperforms the expert's policy and quickly achieves near-optimal performance. RBT, on the other hand, has a good jump-start but then seems to worsen its performance. Once again, this is because each source task has at least few samples that can be transferred. However, selecting such samples is very complicated and leads to negative transfer in case of failure. Finally, FQI performs significantly worse than all alternatives and is, thus, not reported.

\begin{figure}[t]
\center
        \includegraphics[height=3cm]{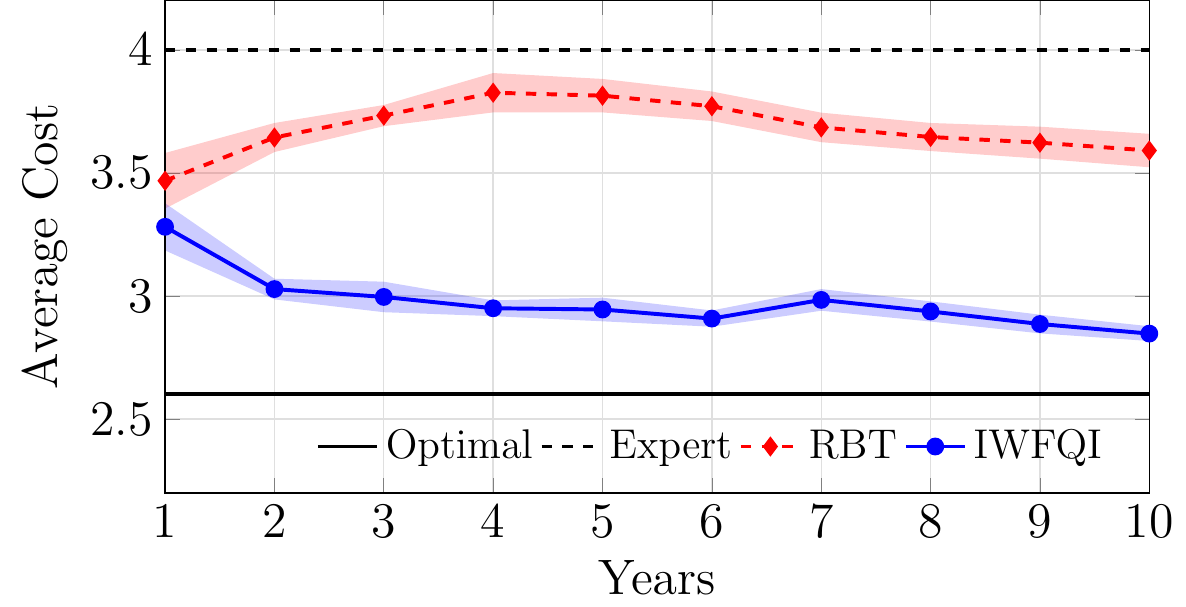}
        \caption{Water reservoir control. Average \emph{cost} per day during the first $10$ years of learning. IWFQI outperforms the expert and quickly achieves near-optimal performance.
        }
        \label{fig:dam}
\end{figure}

\section{Conclusions}

In this paper, we presented Importance Weighted Fitted Q-Iteration, a novel AVI algorithm for transferring samples in batch RL that uses importance weighting to automatically account for the difference in the source and target distributions. 
IWFQI exploits Gaussian processes to learn transition and reward models that are used to compute the importance weights. 
The use of two different processes for reward and transition models allows maximizing the information transferred.
We theoretically investigated IWFQI showing (i) its asymptotic correctness in general settings, and (ii) how to compute a robust statistical estimate of the weights for Gaussian models.
Finally, we empirically proved its effectiveness in common benchmarks and on a real-world water control problem.

One of the drawbacks of our method is that it does not fully exploit possible similarities between tasks. Recent approaches~\citep[\eg][]{doshi2016hidden,killian2017robust} learn models relating a family of tasks to ease the transfer of knowledge. 
Exploring how such relations can benefit our approach (\eg to improve the weight estimates) is an interesting line for future developments.

\clearpage
\bibliography{transfer}
\bibliographystyle{icml2018}

\newpage
\appendix
\onecolumn

\section{Gaussian Models: remarks and proofs}\label{A:gaussianmodels.rewardweights}
In this section we consider the case where the transition and reward models are Gaussian distributions.

A task $M_j = \langle\mathcal{S}, \mathcal{A}, \mathcal{P}_j, \mathcal{R}_j, \gamma \rangle$ is jointly Gaussian when it has:
\begin{itemize}
        \item \emph{Gaussian Reward}: for any $(s,a) \in \mathcal{S}\times\mathcal{A}$
                \begin{equation}\label{E:assumption.gaussianr}
                        \mathcal{R}_j(\cdot|s,a) = \mathcal{N}(\mu_{r}^{(j)}(s,a), \sigma^2_j(s,a)).
                \end{equation}
        \item \emph{Gaussian Transitions}: for any $(s,a) \in \mathcal{S}\times\mathcal{A}$
                \begin{equation}\label{E:assumption.gaussianp}
                        \mathcal{P}_j(\cdot|s,a) = \mathcal{N}(\boldsymbol{\mu}^{(j)}_p(s,a), \bm{\Sigma}_j(s,a))
                \end{equation}.
\end{itemize}
In this setting, by using a Gaussian process to estimate the transition and reward models, we can compute the expected importance weights, where the expectation is taken under the distribution induced by the GPs' predictions, in closed form.
Note that, since reward and transition weights are independent, we can independently consider Assumptions~\ref{E:assumption.gaussianr} and~\ref{E:assumption.gaussianp} for the computation of $w_r$ and $w_p$, respectively.

\thmone*
\begin{proof}
In order to simplify our notation, we consider the dependence on $(s,a,r)$ implicit. We have:
\begin{equation*}
\begin{aligned}
\mathbb{E}\left[ w \right] 
=& \int\int w 
\; \mathcal{N}\left(\wb{r}^{(0)} \big| \mu_{GP_0}, \sigma_{GP_0}^2 \right)
\; \mathcal{N}\left(\wb{r}^{(j)} \big| \mu_{GP_j}, \sigma_{GP_j}^2 \right)
\mathrm{d}\wb{r}^{(0)} \mathrm{d}\wb{r}^{(j)}\\
=& \int\int w
\frac{\exp\left(-\frac{(\wb{r}^{(0)} - \mu_{GP_0})^2}{2\sigma_{GP_0}^2}\right)}{\sqrt{2\pi\sigma_{GP_0}^2}}
\frac{\exp\left(-\frac{(\wb{r}^{(j)} - \mu_{GP_j})^2}{2\sigma_{GP_j}^2}\right)}{\sqrt{2\pi\sigma_{GP_j}^2}}
\mathrm{d}\wb{r}^{(0)} \mathrm{d}\wb{r}^{(j)} \\
=&
\int \frac{\exp\left(-\frac{(r-\wb{r}^{(0)})^2}{2\sigma_0^2}\right)}{\sqrt{2\pi\sigma_0^2}}
\frac{\exp\left(-\frac{(\wb{r}^{(0)} - \mu_{GP_0})^2}{2\sigma_{GP_0}^2}\right)}{\sqrt{2\pi\sigma_{GP_0}^2}} \mathrm{d}\wb{r}^{(0)}
\int \frac{\sqrt{2\pi\sigma_j^2}}{\exp\left(-\frac{(r-\wb{r}^{(j)})^2}{2\sigma_j^2}\right)}
\frac{\exp\left(-\frac{(\wb{r}^{(j)} - \mu_{GP_j})^2}{2\sigma_{GP_j}^2}\right)}{\sqrt{2\pi\sigma_{GP_j}^2}} \mathrm{d}\wb{r}^{(j)}.
\end{aligned}
\end{equation*}
The first integral is over the product of two Gaussian densities, which is known to be~\citep{bromiley2003products}:
\begin{equation}
        \mathcal{N}\left(\wb{r}^{(0)}\big|\bar{\mu}_0,\bar{\sigma}_0^2\right)\mathcal{N}\left(r \big| \mu_{GP_0},\sigma_0^2+\sigma_{GP_0}^2\right),
\end{equation}
where the values of the mean $\bar{\mu}_0$ and variance $\bar{\sigma}_0^2$ of the first density are not important to complete this proof since such density integrates out. By adopting the same procedure as the one described in \cite{bromiley2003products}, we can write the ratio of the two Gaussian densities in the second integral as:
\begin{equation}
        \frac{\sigma_j^2}{\sigma_j^2-\sigma_{GP_j}^2}\frac{\mathcal{N}\left(\wb{r}^{(j)} \big| \bar{\mu}_j,\bar{\sigma}_j^2\right)}{\mathcal{N}\left(r \big| \mu_{GP_j},\sigma_j^2-\sigma_{GP_j}^2\right)},
\end{equation}
where, again, the values of $\bar{\mu}_j$ and $\bar{\sigma}_j^2$ are not relevant to our proof. Finally, we can write:
\begin{equation*}
\begin{split}
\mathbb{E}\left[w(r)\right] =& \int
\mathcal{N}\left(\wb{r}^{(0)} \big| \bar{\mu}_0,\bar{\sigma}_0^2\right)\mathcal{N}\left(r \big| \mu_{GP_0},\sigma_0^2+\sigma_{GP_0}^2\right) \mathrm{d}\wb{r}^{(0)}
\int \frac{\sigma_j^2}{\sigma_j^2-\sigma_{GP_j}^2}\frac{\mathcal{N}\left(\wb{r}^{(j)} \big| \bar{\mu}_j,\bar{\sigma}_j^2\right)}{\mathcal{N}\left(r \big| \mu_{GP_j},\sigma_j^2-\sigma_{GP_j}^2\right)} \mathrm{d}\wb{r}^{(j)} \\
=&
\frac{\sigma_j^2}{\sigma_j^2-\sigma_{GP_j}^2}\frac{\mathcal{N}\left(r\big|\mu_{GP_0},\sigma_0^2+\sigma_{GP_0}^2\right)}{\mathcal{N}\left(r\big|\mu_{GP_j},\sigma_j^2-\sigma_{GP_j}^2\right)}\int
\mathcal{N}\left(\wb{r}^{(0)} \big| \bar{\mu}_0,\bar{\sigma}_0^2\right) \mathrm{d}\wb{r}^{(0)}
\int \mathcal{N}\left(\wb{r}^{(j)} \big| \bar{\mu}_j,\bar{\sigma}_j^2\right) \mathrm{d}\wb{r}^{(j)} \\
=&
\frac{\sigma_j^2}{\sigma_j^2-\sigma_{GP_j}^2} \frac{\mathcal{N}(r\big| \mu_{GP_0},\sigma_0^2+\sigma_{GP_0}^2)}{\mathcal{N}(r \big| \mu_{GP_j},\sigma_j^2-\sigma_{GP_j}^2)}.
\end{split}
\end{equation*}
\end{proof}

We can derive a similar result for the transition model by considering Assumption~\ref{E:assumption.gaussianp}.

\begin{theorem}[Transition Weights in Gaussian Models]
Assume each task to have Gaussian transition distribution $\mathcal{P}_j(\cdot|s,a) = \mathcal{N}(\boldsymbol{\mu}^{(j)}_p(s,a), \bm{\Sigma}_j(s,a))$ with unknown mean. Furthermore, suppose that $\bm{\Sigma}_j(s,a) = diag\left( \delta_{j,1}^2(s,a),\delta_{j,2}^2(s,a),\dots,\delta_{j,D}^2(s,a)  \right)$.
Given the available samples in $\wt{\mathcal{D}}$, we build an estimate of the transition distribution such that, for any MDP $M_j$ and state component $d$, $\wb{p}_d^{(j)}(s,a) \sim \mathcal{N}(\mu_{GP_{j,d}}(s,a),\sigma_{GP_{j,d}}^2(s,a))$.
Then, given a sample $\langle s,a,s' \rangle$ from the $j$-th MDP, the importance weight for the transition model given by $w= \prod_{d=1}^D \frac{\mathcal{N}\left(s_d' \big| \wb{p}_{d}^{(0)}(s,a),\delta_{0,d}^2(s,a)\right) }{\mathcal{N}\left(s_d' \big| \wb{p}_{d}^{(j)}(s,a),\delta_{j,d}^2(s,a)\right)} $ is distributed according to some distribution $\mathcal{G}$ induced by the GPs' predictions.
Let $C_d = \frac{\delta_{j,d}^2(s,a)}{\delta_{j,d}^2(s,a)-\sigma_{GP_{j,d}}^2(s,a)}$ and suppose  $\sigma_{GP_{j,d}}^2(s,a) < \delta_{j,d}^2(s,a)$ for all $d$, then
\begin{equation}\label{eq:exp-weights-p}
\mathbb{E}_{\mathcal{G}}\left[w\right] = \prod_{d=1}^D
C_d\ \frac{\mathcal{N}\left(s_d' \big| \mu_{GP_{0,d}}(s,a),\delta_{0,d}^2(s,a) + \sigma_{GP_{0,d}}^2(s,a)\right)}{\mathcal{N}\left(s_d' \big| \mu_{GP_{j,d}}(s,a),\delta_{j,d}^2(s,a) - \sigma_{GP_{j,d}}^2(s,a)\right)}.
\end{equation}
\end{theorem}
\begin{proof}
Notice that the density $\mathcal{P}_j(\cdot|s,a)$ decomposes into:
\begin{equation}
\mathcal{P}_j(\cdot|s,a) = \prod_{d=1}^D \mathcal{N}\left(\cdot \big| \mu_{p,d}^{(j)}(s,a),\delta_{j,d}^2(s,a)\right).
\end{equation}
By noticing that the $d$-th GP for task $j$ provides an independent estimate of the transition mean's $d$-th component, $\wb{p}_d^{(j)}$, the proof of Theorem \ref{th:gauss-w} can be straightforwardly applied to each component of $s'$, thus proving the theorem.
\end{proof}

\section{Proofs of Theorems}\label{A:proofs}

\thmthree*
\begin{proof}
Applying H\"older's inequality, for all $f\in\mathcal{H}$:
\begin{equation}
\mathbb{E}_Q\left[\left(\widetilde{w}(X,Y)|f(X)-Y|^2\right)^2\right] \leq 16F_{\max}^4\mathbb{E}_Q[\widetilde{w}(X,Y)^2] < \infty.
\end{equation}
Thus, by applying Corollary 1 of \citep{cortes2010learning} to $L_h(x,y) = \widetilde{w}(x,y)|h(x)-y|^2$, we can write:
\begin{equation} \label{eq:bound-h}
\mathbb{E}_Q[\widetilde{w}(X,Y)|\wh{h}(X)-Y|^2] \leq \frac{1}{N}\sum_{i=1}^N \widetilde{w}(x_i,y_i)|\wh{h}(x_i)-y_i|^2 + \\
2^{13/4}F_{\max}^2 \sqrt{\mathbb{E}_Q\left[\widetilde{w}(X,Y)^2\right]} \sqrt[\frac{8}{3}]{\frac{d\log\frac{2Ne}{d} + \log\frac{4}{\delta}}{N}}.
\end{equation}
Let us now expand the left-hand side of \eqref{eq:bound-h}:
\begin{equation} \label{eq:expansion}
\begin{split}
\mathbb{E}_Q\left[\widetilde{w}(X,Y)|\wh{h}(X)-Y|^2\right] =\ & 
\mathbb{E}_\mu\left[ \mathbb{E}_q \left[ \widetilde{w}(X,Y)\left( \wh{h}^2(X) + Y^2 - 2\wh{h}(X)Y \right) \mid X \right] \right] \\ =\ &
\mathbb{E}_\mu\left[ \wh{h}^2(X)\mathbb{E}_q[\widetilde{w}(X,Y) \mid X] + \mathbb{E}_q[ \widetilde{w}(X,Y)Y^2 \mid X] -2\wh{h}(X)\mathbb{E}_q[ \widetilde{w}(X,Y)Y \mid X]\right.\\ &\left.  \pm \mathbb{E}_q^2[ \widetilde{w}(X,Y)Y \mid X] \pm \wh{h}^2(X) \right] \\ =\ &
\mathbb{E}_\mu\left[ \left( \wh{h}(X) - \mathbb{E}_q[ \widetilde{w}(X,Y)Y \mid X] \right)^2 + \wh{h}^2(X)\mathbb{E}_q[\widetilde{w}(X,Y) \mid X] + \mathbb{E}_q[ \widetilde{w}(X,Y)Y^2 \mid X]\right.\\ &\left. - \mathbb{E}_q^2[ \widetilde{w}(X,Y)Y \mid X] - \wh{h}^2(X) \right] \\ =\ &
\|\wh{h}-\widetilde{h}\|_\mu^2 + \mathbb{E}_\mu[\wh{h}^2(X)\left(\mathbb{E}_q[\widetilde{w}(X,Y)-1 \mid X]\right)] + K,
\end{split}
\end{equation}
where $K = \mathbb{E}_\mu[\mathbb{E}_q[ \widetilde{w}(X,Y)Y^2 \mid X] - \mathbb{E}_q^2[ \widetilde{w}(X,Y)Y \mid X]]$ is a constant term (independent of $\wh{h}$) and $\widetilde{h}(x) = \mathbb{E}_q[ \widetilde{w}(x,Y)Y \mid x]$ is the regression function weighted by $\widetilde{w}$. Plugging this into \eqref{eq:bound-h} we get:
\begin{equation} \label{eq:dir1}
\begin{split}
\|\wh{h}-\widetilde{h}\|_\mu^2 + \mathbb{E}_\mu[\wh{h}^2(X)\left(\mathbb{E}_q[\widetilde{w}(X,Y)-1 \mid X]\right)] + K \leq \\ \frac{1}{N}\sum_{i=1}^N \widetilde{w}(x_i,y_i)|\wh{h}(x_i)-y_i|^2 +
2^{13/4}F_{max}^2 \sqrt{\mathbb{E}_Q[\widetilde{w}(X,Y)]} \sqrt[\frac{8}{3}]{\frac{d\log\frac{2Ne}{d} + \log\frac{4}{\delta}}{N}}.
\end{split}
\end{equation}
Consider now the hypothesis $h_0\in\mathcal{H}$ such that $h_0 = \underset{f\in\mathcal{H}}{argmin}\ \|f-\widetilde{h}\|_\mu^2$. Since $h_0$ is in $\mathcal{H}$ and $\wh{h}$ was defined as the hypothesis minimizing the empirical weighted loss, we have:
\begin{equation} \label{eq:empirical}
\frac{1}{N}\sum_{i=1}^N \widetilde{w}(x_i,y_i)|\wh{h}(x_i)-y_i|^2 \leq \frac{1}{N}\sum_{i=1}^N \widetilde{w}(x_i,y_i)|h_0(x_i)-y_i|^2.
\end{equation}
Similarly to what we did for $\wh{h}$, we can bound the empirical error of $h_0$. According to Corollary 1 of \citep{cortes2010learning}, we have that for any $\delta>0$, with probability at least $1-\delta$:
\begin{equation}
\frac{1}{N}\sum_{i=1}^N \widetilde{w}(x_i,y_i)|h_0(x_i)-y_i|^2 \leq \mathbb{E}_Q[\widetilde{w}(X,Y)|h_0(X)-Y|^2] + \\
2^{13/4}F_{max}^2 \sqrt{\widehat{\mathbb{E}}_{\mathcal{D}}[\widetilde{w}(X,Y)]} \sqrt[\frac{8}{3}]{\frac{d\log\frac{2Ne}{d} + \log\frac{4}{\delta}}{N}}.
\end{equation}
By adopting \eqref{eq:expansion} to expand the expected error of $h_0$, we obtain:
\begin{equation} \label{eq:dir2}
\begin{split}
\frac{1}{N}\sum_{i=1}^N \widetilde{w}(x_i,y_i)|h_0(x_i)-y_i|^2 \leq \\ \underset{f\in\mathcal{H}}{inf}\|f-\widetilde{h}\|_\mu^2 + \mathbb{E}_\mu[h_0^2(X)\left(\mathbb{E}_q[\widetilde{w}(X,Y)-1 \mid X]\right)] + K +
2^{13/4}F_{max}^2 \sqrt{\widehat{\mathbb{E}}_{\mathcal{D}}[\widetilde{w}(X,Y)]} \sqrt[\frac{8}{3}]{\frac{d\log\frac{2Ne}{d} + \log\frac{4}{\delta}}{N}}.
\end{split}
\end{equation}
If we now put \eqref{eq:dir1} and \eqref{eq:dir2} together by means of \eqref{eq:empirical}, we get that, with probability at least $1-2\delta$:
\begin{equation} \label{eq:bnd1}
\begin{split}
\|\wh{h}-\widetilde{h}\|_\mu^2 \leq\ & \underset{f\in\mathcal{H}}{inf}\|f-\widetilde{h}\|_\mu^2 + \mathbb{E}_\mu[\left(h_0^2(X)-\wh{h}^2(X)\right)\left(\mathbb{E}_q[\widetilde{w}(X,Y)-1 \mid X]\right)] \\
                                     &\quad{}+
2^{13/4}F_{max}^2\left(\sqrt{\mathbb{E}_Q[\widetilde{w}(X,Y)^2]} + \sqrt{\widehat{\mathbb{E}}_{\mathcal{D}}[\widetilde{w}(X,Y)^2]}\right)\sqrt[\frac{8}{3}]{\frac{d\log\frac{2Ne}{d} + \log\frac{4}{\delta}}{N}} \\ \leq\ &
\underset{f\in\mathcal{H}}{inf}\|f-\widetilde{h}\|_\mu^2 + F_{max}^2\mathbb{E}_\mu\left[\left\vert\mathbb{E}_q[\widetilde{w}(X,Y) \mid X] - 1 \right\vert\right] \\
&\quad{}+
2^{13/4}F_{max}^2\left(\sqrt{\mathbb{E}_Q[\widetilde{w}(X,Y)^2]} + \sqrt{\widehat{\mathbb{E}}_{\mathcal{D}}[\widetilde{w}(X,Y)^2]}\right)\sqrt[\frac{8}{3}]{\frac{d\log\frac{2Ne}{d} + \log\frac{4}{\delta}}{N}} \\ =\ &
\underset{f\in\mathcal{H}}{inf}\|f-\widetilde{h}\|_\mu^2 + F_{max}^2 \|g\|_{1,\mu} + 2^{13/4}F_{max}^2 M(\wt{w})\left(\frac{d\log\frac{2Ne}{d} + \log\frac{4}{\delta}}{N}\right)^{\frac{3}{8}}.
\end{split}
\end{equation}
By taking the square root of both sides of \eqref{eq:bnd1} and using $\sqrt{\sum_i a_i} \leq \sum_i \sqrt{a_i}$ for $a_i \geq 0$, we obtain:
\begin{equation} \label{eq:bnd5}
\|\wh{h}-\widetilde{h}\|_\mu \leq \underset{f\in\mathcal{H}}{inf}\|f-\widetilde{h}\|_\mu + F_{max} \sqrt{\|g\|_{1,\mu}} + 2^{13/8}F_{max} \sqrt{M(\wt{w})}\left(\frac{d\log\frac{2Ne}{d} + \log\frac{4}{\delta}}{N}\right)^{\frac{3}{16}}.
\end{equation}
Furthermore:
\begin{equation} \label{eq:bnd3}
\underset{f\in\mathcal{H}}{inf}\|f-\widetilde{h}\|_\mu \leq \underset{f\in\mathcal{H}}{inf}\|f-h^\star\|_\mu + \|h^\star-\widetilde{h}\|_\mu.
\end{equation}
We can now bound the expected error of $\wh{h}$ with respect to $h^\star$ by:
\begin{equation} \label{eq:bnd4}
\|\wh{h}-h^\star\|_\mu \leq \|\wh{h}-\widetilde{h}\|_\mu + \|\widetilde{h}-h^\star\|_\mu.
\end{equation}
We already provided a bound on the first term, so let us analyze the second one. We have:
\begin{equation} \label{eq:bnd2}
\begin{split}
\|\widetilde{h}-h^\star\|_\mu^2 =\ & \mathbb{E}_\mu[|\widetilde{h}(X)-h^\star(X)|^2] \\ =\ &
\mathbb{E}_\mu[|\mathbb{E}_q[\widetilde{w}(X,Y)Y\mid X]-\mathbb{E}_p[Y\mid X]|^2] \\ =\ &
\mathbb{E}_\mu[|\mathbb{E}_q[\widetilde{w}(X,Y)Y\mid X]-\mathbb{E}_q[w(X,Y)Y\mid X]|^2] \\ =\ &
\mathbb{E}_\mu[|\mathbb{E}_q[Y\left(\widetilde{w}(X,Y)-w(X,Y)\right)\mid X]|^2] \\ \leq\ & 
\mathbb{E}_\mu[\mathbb{E}_q[|Y|^2|\widetilde{w}(X,Y)-w(X,Y)|^2\mid X]] \\ \leq\ &
F_{max}^2\mathbb{E}_Q\left[|\widetilde{w}(X,Y)-w(X,Y)|^2\right] \\ =\ & F_{max}^2\|\widetilde{w}-w\|_Q^2.
\end{split}
\end{equation}
where the first inequality is due to Jensen's inequality. Thus, $\|\widetilde{h}-h^\star\|_\mu \leq F_{max}\|\widetilde{w}-w\|_Q$. By combining \eqref{eq:bnd5}, \eqref{eq:bnd3}, \eqref{eq:bnd4}, and \eqref{eq:bnd2}, we have:
\begin{align*}
\|\wh{h}-h^\star\|_\mu \leq\ & \|\wh{h}-\widetilde{h}\|_\mu + \|\widetilde{h}-h^\star\|_\mu\\ \leq\ & \underset{f\in\mathcal{H}}{inf}\|f-\widetilde{h}\|_\mu + F_{max} \sqrt{\|g\|_{1,\mu}} + 2^{13/8}F_{max} \sqrt{M(\wt{w})}\left(\frac{d\log\frac{2Ne}{d} + \log\frac{4}{\delta}}{N}\right)^{\frac{3}{16}} + \|\widetilde{h}-h^\star\|_\mu\\ \leq\ & \underset{f\in\mathcal{H}}{inf}\|f-h^\star\|_\mu + F_{max} \sqrt{\|g\|_{1,\mu}} + 2^{13/8}F_{max} \sqrt{M(\wt{w})}\left(\frac{d\log\frac{2Ne}{d} + \log\frac{4}{\delta}}{N}\right)^{\frac{3}{16}} + 2\|\widetilde{h}-h^\star\|_\mu\\ \leq\ & \underset{f\in\mathcal{H}}{inf}\|f-h^\star\|_\mu + F_{max} \sqrt{\|g\|_{1,\mu}} + 2^{13/8}F_{max} \sqrt{M(\wt{w})}\left(\frac{d\log\frac{2Ne}{d} + \log\frac{4}{\delta}}{N}\right)^{\frac{3}{16}} + 2F_{max}\|\widetilde{w}-w\|_Q
\end{align*}
which concludes the proof.
\end{proof}

\begin{lemma}\label{lemma:err-r}
Let $\mathcal{H} \subset B(\mathcal{S}\times\mathcal{A},Q_{max})$ be a functional space. Call $g_r(s,a) = \mathbb{E}_{\mathcal{R}_S}[\widetilde{w}_r(r|s,a)]-1$ and $M(\widetilde{w}_r)=\sqrt{\mathbb{E}_{\phi_S^R}[\widetilde{w}_r(r|s,a)^2]} + \sqrt{\widehat{\mathbb{E}}_{\mathcal{D}}[\widetilde{w}_r(r|s,a)^2]}$, where $\phi^R_S(r|s,a) = \mu(s,a)\mathcal{R}_s(r|s,a)$ and $\mathcal{D}$ is a dataset of $N$ i.i.d. samples. Assume $d = Pdim(\{ |f(s,a)-r|^2\ : f\in\mathcal{H} \}) < \infty$ and $\mathbb{E}_{\phi_S^R}[\widetilde{w}_r(r|s,a)^2] < \infty$. Let $\widehat{R}$ be as defined in \eqref{eq:rhat}. Then, for any $\delta>0$, with probability at least $1-2\delta$:
\begin{equation}
\begin{split}
||R-\widehat{R}||_\mu \leq\ &
\underset{f\in\mathcal{H}}{inf}\ ||f-R||_\mu + 
r_{max}\sqrt{||g_r||_{1,\mu}} \\ +\ &
2^{13/8}Q_{max} \sqrt{M(\widetilde{w}_r)}\left(\frac{d\log\frac{2Ne}{d} + \log\frac{4}{\delta}}{N}\right)^{\frac{3}{16}} \\ +\ &
2r_{max}||\widetilde{w}_r-w_r||_{\phi_S^R}.
\end{split}
\end{equation}
\end{lemma}
\begin{proof}
The result follows straightforwardly by applying Theorem \ref{th:regr}.
\end{proof}

\begin{lemma}\label{lemma:err-q}
Let $\mathcal{H} \subset B(\mathcal{S}\times\mathcal{A},Q_{max})$ be a functional space. Call $g_p(s,a) = \mathbb{E}_{\mathcal{P}_S}[\widetilde{w}_p(s'|s,a)]-1$ and $M(\widetilde{w}_p)=\sqrt{\mathbb{E}_{\phi_S^P}[\widetilde{w}_p(s'|s,a)^2]} + \sqrt{\widehat{\mathbb{E}}_{\mathcal{D}}[\widetilde{w}_p(s'|s,a)^2]}$, where $\phi^P_S(r|s,a) = \mu(s,a)\mathcal{P}_s(r|s,a)$ and $\mathcal{D}$ is a dataset of $N$ i.i.d. samples. Assume $d = Pdim(\{ |f(s,a)-r|^2\ : f\in\mathcal{H} \}) < \infty$ and $\mathbb{E}_{\phi_S^P}[\widetilde{w}_p(s'|s,a)^2] < \infty$. Let $Q_{k+1}$ be as defined in \eqref{eq:qk} and denote $\wt{L}^*Q(s,a) \coloneqq \wh{R}(s,a) + \int_{\mathcal{S}} \mathcal{P}_T(\mathrm{d}s'|s,a) \max_{a'} Q(s',a)$. Then, for any $\delta>0$, with probability at least $1-2\delta$:
\begin{equation}
\begin{split}
||\wt{L}^*Q_k - Q_{k+1}||_\mu \leq\ &
\underset{f\in\mathcal{H}}{inf}\ ||f-\wt{L}^*Q_k||_\mu + 
Q_{max}\sqrt{||g_p||_{1,\mu}} \\ +\ &
2^{13/8}Q_{max} \sqrt{M(\widetilde{w}_p)}\left(\frac{d\log\frac{2Ne}{d} + \log\frac{4}{\delta}}{N}\right)^{\frac{3}{16}} \\ +\ &
2Q_{max}||\widetilde{w}_p-w_p||_{\phi_S^P}.
\end{split}
\end{equation}
\end{lemma}
\begin{proof}
The result follows straightforwardly by applying Theorem \ref{th:regr}.
\end{proof}

\thmfour*
\begin{proof}
We can decompose the error at iteration $k$ into:
\begin{align}
\|\epsilon_k\|_\mu 
&= \|L^* Q_k - Q_{k+1}\|_\mu \notag \\ 
&\leq \|L^* Q_k - \wt{L}^*Q_k\|_\mu + \|\wt{L}^*Q_k - Q_{k+1}\|_\mu \notag \\ 
&= \|R - \widehat{R}\|_\mu + \|\wt{L}^*Q_k - Q_{k+1}\|_\mu, \label{PE:epsilonk.decomposition}
\end{align}
where, for any pair $(s,a)$, $\wt{L}^*Q(s,a) \coloneqq \wh{R}(s,a) + \int_{\mathcal{S}} \mathcal{P}_T(\mathrm{d}s'|s,a) \max_{a'} Q(s',a)$ is the optimal Bellman operator of the target task using the approximated reward function defined in~\eqref{eq:rhat}.
The two terms in \eqref{PE:epsilonk.decomposition} can be bounded straightforwardly by applying Lemma \ref{lemma:err-r} and Lemma \ref{lemma:err-q}, respectively.
The application of Lemma \ref{lemma:err-q} to the second term gives rise to 
$\inf_{f\in\mathcal{H}} \|f-\wt{L}^* Q_k\|_\mu$, which can be further bounded by noticing that:
\begin{equation}\label{PE:reg_error.approximateLop}
\begin{split}
        \inf_{f\in\mathcal{H}} \|f-\wt{L}^* Q_k\|_\mu 
\leq & \inf_{f\in\mathcal{H}}\ \|f-L^* Q_k\|_\mu + \|L^* Q_k - \wt{L}^* Q_k\|_\mu.
\end{split}
\end{equation}
The second term in~\eqref{PE:reg_error.approximateLop} is again $\|R - \widehat{R}\|_\mu$, while the first term has already been bounded in Theorem 5.3 of \cite{farahmand2011regularization}:
\begin{equation}\label{eq:prop}
\begin{split}
        \inf_{f\in\mathcal{H}} \|f-L^* Q_k\|_\mu 
        \leq& \inf_{f\in\mathcal{H}} \|f-(L^*)^{k+1}Q_0\|_\mu + \sum_{i=0}^{k-1} (\gamma C_{AE}(\mu))^{i+1}\|\epsilon_{k-i-1}\|_\mu.
\end{split}
\end{equation}
Then, by combining the bounds from Lemma \ref{lemma:err-r} and Lemma \ref{lemma:err-q} with \eqref{PE:reg_error.approximateLop} and \eqref{eq:prop}, we can write:
\begin{align*}
\|\epsilon_k\|_\mu \leq\ & \|R - \widehat{R}\|_\mu + \|\wt{L}^*Q_k - Q_{k+1}\|_\mu\\ \leq\ & \|R - \widehat{R}\|_\mu + \underset{f\in\mathcal{H}}{inf}\ ||f-\wt{L}^*Q_k||_\mu + Q_{max}\sqrt{||g_p||_{1,\mu}}\\ &+ 2^{13/8}Q_{max} \sqrt{M(\widetilde{w}_p)}\left(\frac{d\log\frac{2Ne}{d} + \log\frac{4}{\delta}}{N}\right)^{\frac{3}{16}} +
2Q_{max}||\widetilde{w}_p-w_p||_{\phi_S^P}\\ \leq\ & 
2\|R - \widehat{R}\|_\mu + \inf_{f\in\mathcal{H}} \|f-(L^*)^{k+1}Q_0\|_\mu + \sum_{i=0}^{k-1} (\gamma C_{AE}(\mu))^{i+1}\|\epsilon_{k-i-1}\|_\mu + Q_{max}\sqrt{||g_p||_{1,\mu}}\\ &+ 2^{13/8}Q_{max} \sqrt{M(\widetilde{w}_p)}\left(\frac{d\log\frac{2Ne}{d} + \log\frac{4}{\delta}}{N}\right)^{\frac{3}{16}} +
2Q_{max}||\widetilde{w}_p-w_p||_{\phi_S^P}\\ \leq\ & 
Q_{\max}\sqrt{\|g_p\|_{1,\mu}} + 2\rmax \sqrt{\|g_r\|_{1,\mu}} + 2Q_{\max}\|\wt{w}_p-w_p\|_{\phi^P_S} + 4\rmax\|\wt{w}_r-w_r\|_{\phi^R_S} \\
&+ \inf_{f \in \mathcal{H}} \|f-(L^*)^{k+1}Q_0\|_{\mu} + 2 \inf_{f \in \mathcal{H}} \|f-R\|_{\mu} \\ 
&+ 2^{\frac{13}{8}}Q_{\max}\left(\sqrt{M(\wt{w}_p)}+2\sqrt{M(\wt{w}_r)}\right)
\left(\frac{d\log\frac{2Ne}{d} + \log\frac{4}{\delta}}{N} \right)^{\frac{3}{16}} + \sum_{i=0}^{k-1} (\gamma C_{\textsc{AE}}(\mu))^{i+1}\|\epsilon_{k-i-1}\|_\mu,
\end{align*}
where we recall that $g_p(s,a) = \mathbb{E}_{\mathcal{P}_S}[\wt{w}_p(s'|s,a)|s,a]-1$, $M(\wt{w}_p)=\sqrt{\mathbb{E}_{\phi^P_S}[\wt{w}_p(s'|s,a)^2]} + \sqrt{\wh{\mathbb{E}}_{\mathcal{D}}[\wt{w}_p(s'|s,a)^2]}$, and $\phi^P_S(s'|s,a) = \mu(s,a)\mathcal{P}_S(s'|s,a)$. This concludes the proof.
\end{proof}

\section{Additional Details on the Experiments}

\subsection{Puddle World}\label{app:puddleworld}
Our first experimental domain is a modified version of the puddle world environment presented in \cite{sutton1996generalization}.
Puddle world is a two-dimensional continuous grid with a goal area and some elliptical ``puddles''. The goal is to drive the agent from a starting position to the goal area while avoiding the puddles. The state-space is $[0,10]^2$, while the action-space is discrete and allows the agent to move in the four cardinal directions. At each time-step, the agent receives a reward of $-1$ plus a penalization proportional to the distance from all puddles:
$R(s,a) = -1 - 100\sum_{u\in\mathcal{U}}W_u(s)$,
where $\mathcal{U}$ is the set of puddles and $W_u(s)$ is the weight of puddle $u$ for state $s$. In the goal the reward is zero. In our experiments, we modeled $W_u(s)$ as a bivariate Gaussian.
Each action moves the agent by $\alpha$ in the corresponding direction. In particular, we consider two versions of the environment: I) \emph{shared dynamics} where $\alpha = 1$ and II) \emph{puddle-based dynamics} where puddles also slow-down the agent by: $\alpha = (1 + 5\sum_{u \in \mathcal{U}} W_{u}(s'))^{-1}$.
Finally, a white Gaussian noise  of $\sigma_r^2 = 0.01$ and $\sigma_p^2 = 0.04$ is added to the reward and the transition model, respectively. In our experiments we set $\gamma = 0.99$ and a maximum horizon of $50$ time-steps.

We provide additional details on the puddle world experiments. The target task and the three source tasks can be seen in Figure \ref{fig:pw-tasks}. Notice that the optimal paths to solve each task have at least a small overlapping, thus allowing some knowledge transfer. However, the optimal policy for one task is likely to cross a puddle if carelessly used in another domain. This makes the transfer problem more challenging since the algorithm has to figure out which samples should be retained and which should be discarded.

In both experiments, $20$ episodes were generated beforehand from each source task. For IWFQI, a Gaussian process was fitted on each of the three source datasets using the squared exponential kernel. The noises of the reward and transition models were estimated as $10$ times their true value. In each algorithm, FQI was run for $50$ iterations using Extra-Trees with $50$ estimators and a minimum of $2$ samples to split a node. Results were averaged over $20$ independent runs.

\begin{figure}
  \centering
  \includegraphics[width=.24\linewidth]{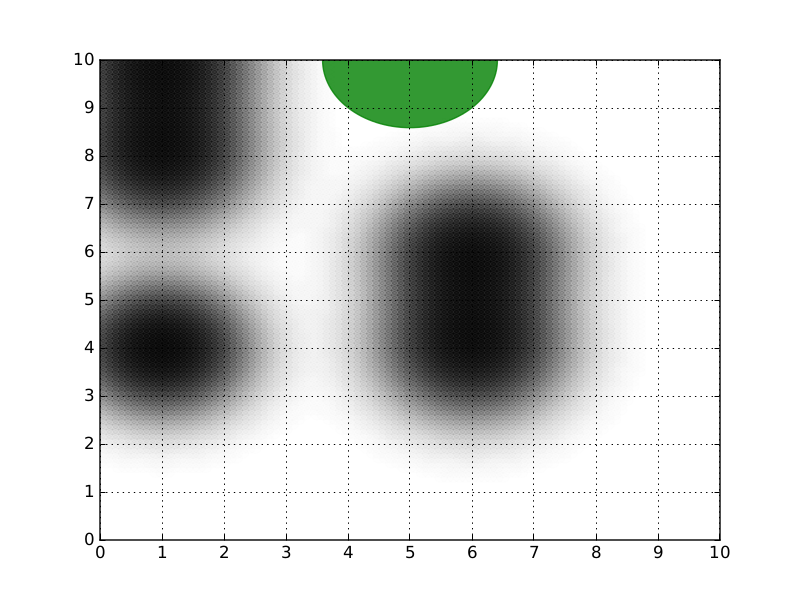}
  \includegraphics[width=.24\linewidth]{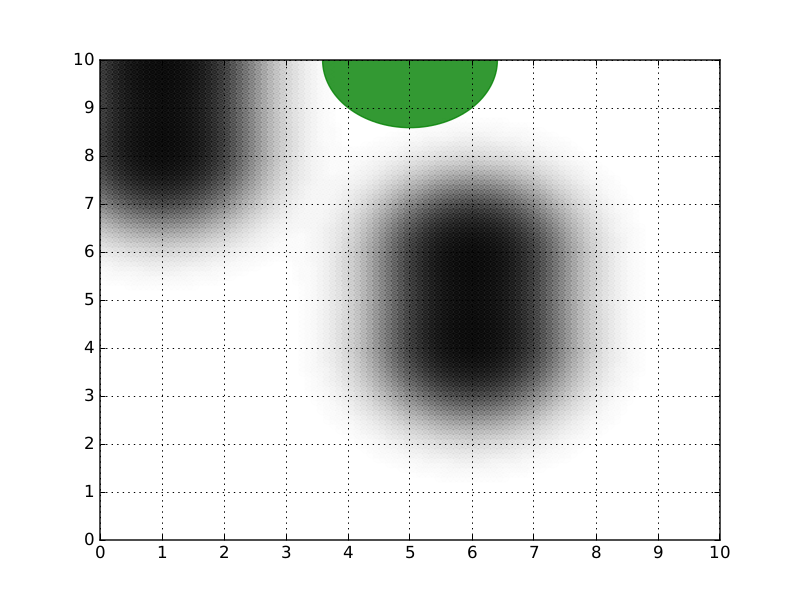}
  \includegraphics[width=.24\linewidth]{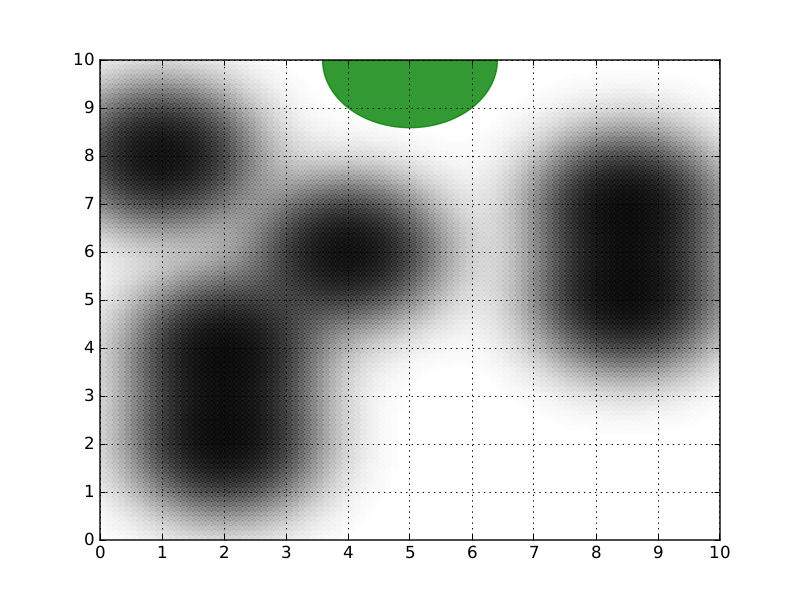}
  \includegraphics[width=.24\linewidth]{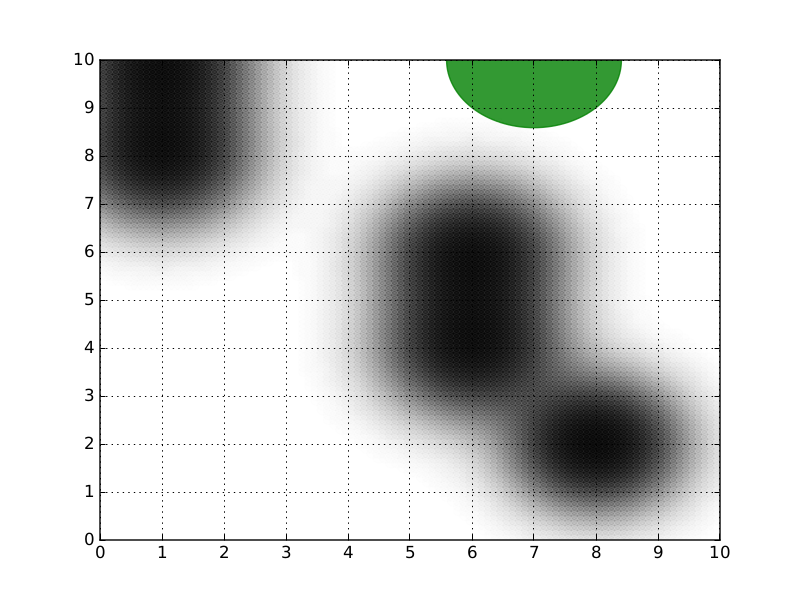}
  \caption{From left to right: the target task and the three source tasks. The agent always starts in the bottom-left corner and must reach the goal area (shown in green). Puddles are shown in black.}
\label{fig:pw-tasks}
\end{figure}

\subsection{Acrobot}\label{app:acrobot}

We provide a precise description of the two tasks used in the Acrobot experiment. For both tasks, the state-space is composed of the two link angles $(\theta_1,\theta_2)$ and their velocities $(\dot{\theta}_1,\dot{\theta}_2)$. The transition dynamics are the ones described in \cite{sutton1998reinforcement}. The agent can only apply a torque of $+2$ or $-2$ to the joint between the two links. The initial state is $(\theta_1,0,0,0)$, where $\theta_1 \sim \mathcal{U}(-2,2)$. Performance is evaluated starting from multiple states $(\theta_1,0,0,0)$, with $\theta_1$ evenly spaced in $[-2,2]$. The swing-up task has reward:
\begin{equation}
R_{sw}(\theta_1,\theta_2,\dot{\theta}_1,\dot{\theta}_2) = -cos(\theta_1) - cos(\theta_1 + \theta_2) - 2,
\end{equation}
and terminates whenever $-cos(\theta_1) - cos(\theta_1 + \theta_2) > 1$ or $100$ time-steps are reached. The constant-spin task has reward:
\begin{equation}
R_{cs}(\theta_1,\theta_2,\dot{\theta}_1,\dot{\theta}_2) = - |\dot{\theta}_1 - \pi|,
\end{equation}
and terminates whenever $100$ time-steps are reached.

We collect $100$ episodes from the first source task (corresponding to $3400$ samples) and $50$ episodes from the second source task (corresponding to $5000$ samples). For all algorithms, FQI uses extra-trees with $50$ estimators and a minimum of $20$ samples to split a node. Data is collected in batches of $10$ episodes using an $\epsilon$-greedy policy ($\epsilon = 0.1$). For IWFQI, GPs use the squared exponential kernel with parameters estimated by maximum likelihood on the data.

To further demonstrate the advantages of our approach, we show what happens when only the constant-spin source task is available. Clearly, most of the reward samples should be discarded, and conversely for the transition samples. As we can see from Figure \ref{fig:ac2-disc} and \ref{fig:ac2-steps}, RBT now performs significantly worse than FQI. This is due to the fact that, by transferring samples jointly, it cannot avoid introducing bias. Our approach, on the other hand, is able to discard the reward samples, thus being robust to negative transfer. Furthermore, it achieves a little improvement over FQI due to the few samples transferred.

\begin{figure}[t]
\centering
    \begin{subfigure}[b]{0.35\textwidth}
        \includegraphics[height=4.3cm]{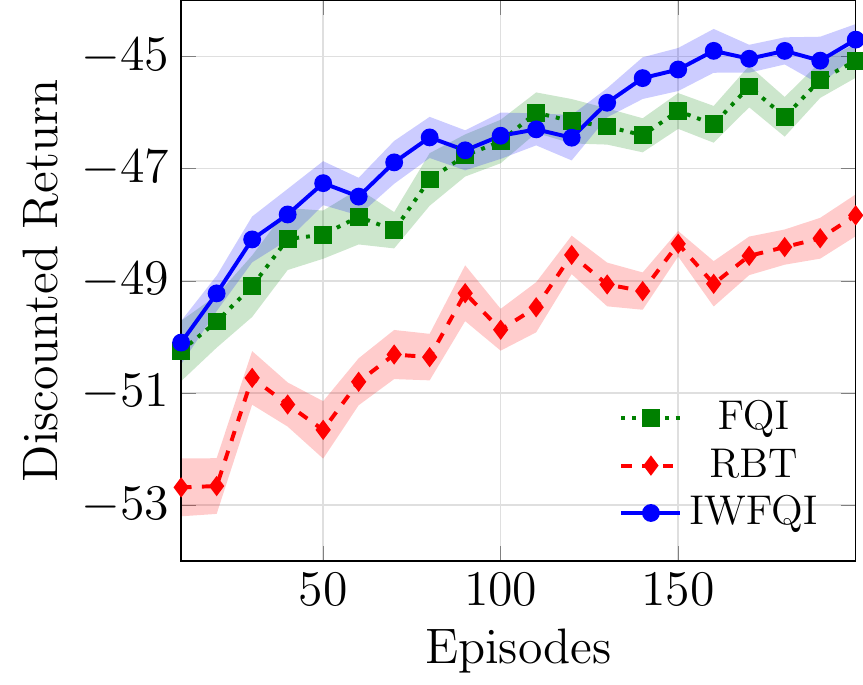}
        \caption{}
        \label{fig:ac2-disc}
    \end{subfigure}
    \begin{subfigure}[b]{0.35\textwidth}
        \includegraphics[height=4.3cm]{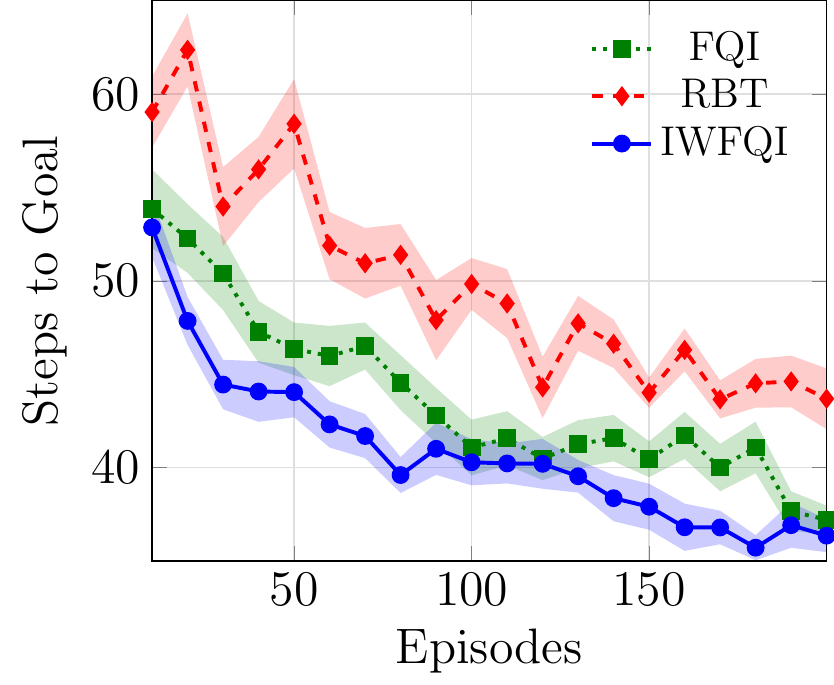}
        \caption{}
        \label{fig:ac2-steps}
    \end{subfigure}
    \caption{Transfer of samples from the constant-spin task to the swing-up task. (a) discounted expected reward and (b) number of steps before reaching a goal state.}
\end{figure}

\subsection{Water Reservoir Control}\label{app:dam}

All tasks used in this experiment are modeled according to the dynamics described in the paper. For the sake of simplicity, each water reservoir is supposed to have capacity of $500\ Mm^3$, minimum storage of $50\ Mm^3$, flooding threshold of $300\ Mm^3$, and per-day demand of $10\ Mm^3$. Due to the different geographic locations, each task has different inflow function $i_j(t) = \bar{i}_j(t) + \mathcal{N}(0,\sigma_p^2)$, where $\sigma_p^2 = 2.0$ is the fixed noise variance. The different mean-inflow functions are shown in Figure \ref{fig:inflow}. Furthermore, each water reservoir weighs the flooding and demand objectives differently. This is modeled by changing the respective weights $\alpha$ and $\beta$. The values for all tasks are reported in Table \ref{tab:coeff}. Notice that there is no source task that is globally similar to the target: either some reward structure is shared or some transition structure is, never both. This makes transfer very challenging since samples have to be accurately selected to prevent detrimental consequences.

In this experiment, we run FQI using extra-trees for $80$ iterations with $100$ estimators and a minimum of $10$ samples to split a node. GPs use the anisotropic squared exponential kernel. For each of the $6$ source reservoirs, we gather $30$ years of historical data where controls are applied by a human operator's policy. We learn the target task by collecting batches of $1$ year, each using an $\epsilon$-greedy policy ($\epsilon=0.3$) on the previously learned Q-function. Evaluation is performed by averaging $10$ trajectories of $1$ year, each starting from January $1^{st}$ and with an initial storage of $200\ Mm^3$ of water.

To better demonstrate the difficulty of this control problem, we run FQI for $500$ episodes (equivalent to $500$ years of interaction). Furthermore, to make the problem simpler, we allow the agent to sample the state-action space arbitrarily, so as to have a better exploration. The result is shown in Figure \ref{fig:dam-fqi}. Although we significantly simplified the problem and we allowed FQI to gather an enormous amount of data, the algorithm still needs almost $500$ years to achieve optimal performance. This demonstrates that solving this task by directly interacting with the real environment is clearly impractical. Thus, transfer of previous knowledge is, in this case, mandatory to achieve good performance.

\begin{table}[h]
  \begin{center}
    \caption{Reward parameters for the different water reservoirs.}
    \label{tab:coeff}
    \begin{tabular}{| l | l | l | l | l | l | l | l |}
    \hline
    Parameter & Target & Source 1 & Source 2 & Source 3 & Source 4 & Source 5 & Source 6 \\ \hline
    $\alpha$ & 0.3 & 0.8 & 0.35 & 0.7 & 0.4 & 0.6 & 0.45 \\ \hline
    $\beta$ & 0.7 & 0.2 & 0.65 & 0.3 & 0.6 & 0.4 & 0.55 \\ \hline
    \end{tabular}
  \end{center}
\end{table}

\begin{figure}[b]
\centering
    \begin{subfigure}[b]{0.4\textwidth}
        \includegraphics[height=4.3cm]{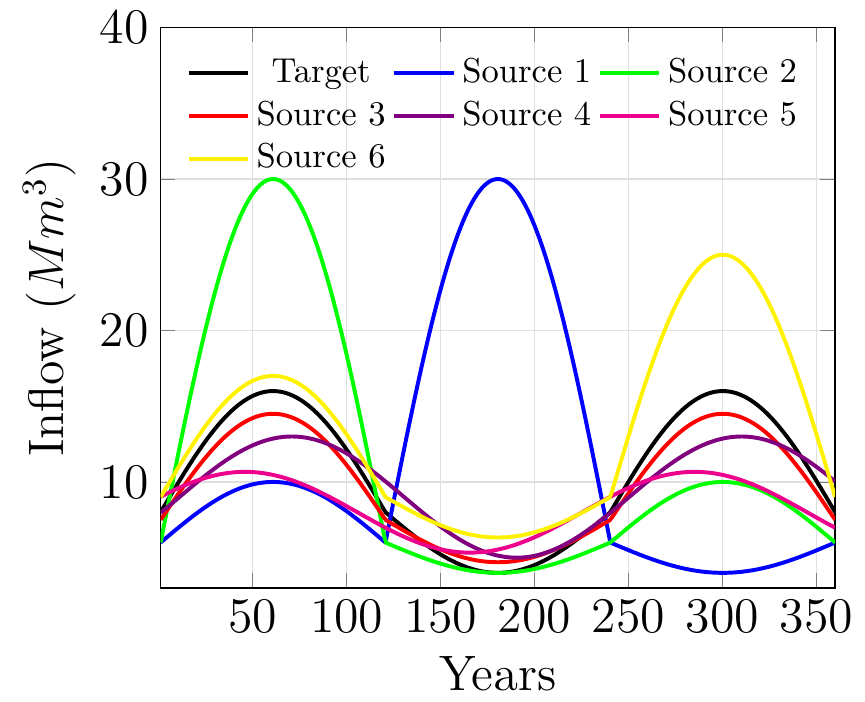}
        \caption{}
        \label{fig:inflow}
    \end{subfigure}
    \begin{subfigure}[b]{0.4\textwidth}
        \includegraphics[height=4.3cm]{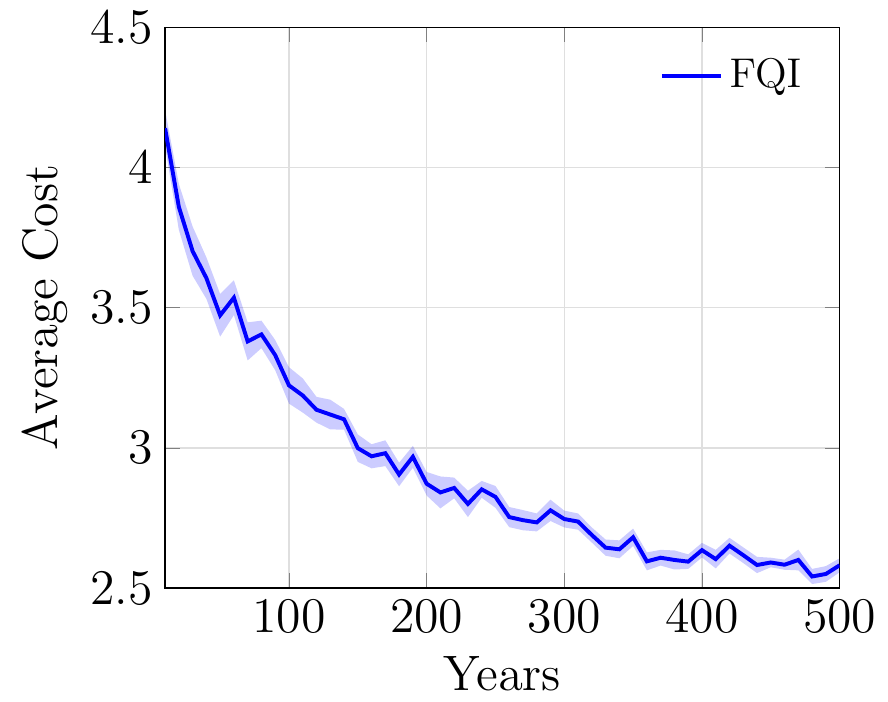}
        \caption{}
        \label{fig:dam-fqi}
    \end{subfigure}
    \caption{Water reservoir control. (a) Inflow profiles for all tasks. (b) Learning without transfer for $500$ years.}
\end{figure}

\end{document}